%% file: reward-main-v2.tex
\documentclass[11pt]{article}
\usepackage{fullpage}
\input{commands}

\begin{document}

\begin{center}
{\bf {\LARGE{Agnostic learning with unknown utilities}}}

\vspace*{.2in}

\large{
\begin{tabular}{cccc}
Kush Bhatia$^\dagger$ &  Peter L. Bartlett$^{\dagger, \ddagger}$ & Anca D. Dragan$^{\dagger}$ & Jacob Steinhardt$^{\ddagger}$
\end{tabular}
}
\vspace*{.2in}

\begin{tabular}{c}
Department of Electrical Engineering and Computer Sciences, UC
Berkeley$^\dagger$ \\ Department of Statistics, UC Berkeley$^\ddagger$
\end{tabular}

\vspace*{.2in}

\today

\vspace*{.2in}

\end{center}

\input{abs}
\input{intro}
\input{prob}
\input{results}
\input{bin-class}
\input{local-minimax}
\input{acks}
\newpage
\appendix
\appendixpage
\input{rel}
\input{app_defproofA}
\input{app_defproofB}
\newpage
\printbibliography
\end{document}

%% file: commands.tex
\usepackage[dvipsnames]{xcolor}

\newcommand{\red}[1]{\textcolor{red}{#1}}

\newcommand{\green}[1]{\textcolor{orange}{#1}}

\newcommand{\jscomm}[1]{{\bf{{\red{{JS -- #1}}}}}}

\newcommand{\adnote}[1]{{{{\green{{AD -- #1}}}}}}

\usepackage{appendix}
\usepackage{epsf}

\usepackage[dvipsnames]{xcolor}
\usepackage{xspace}
\usepackage{graphics}
\usepackage[normalem]{ulem}
\usepackage[style = alphabetic,citestyle=alphabetic,maxbibnames=99,backend=biber,natbib=true,backref=true]{biblatex}
\DefineBibliographyStrings{english}{%
  backrefpage = {Cited on page},
  backrefpages = {Cited on pages},
}
\usepackage{nicefrac}
\usepackage{subcaption}
\usepackage{psfrag}
\usepackage{comment}
\usepackage{color}
\usepackage{wrapfig}
\usepackage{pgfplots}
\usepackage{pgfplotstable}
\usepackage{multirow}
\usepackage{mathtools}
\usepackage{amsfonts}
\usepackage{amsthm}
\usepackage{amsmath}
\usepackage{amssymb}
\usepackage{scalerel}
\usepackage{nccmath}
\usepackage{tikz}
\usepackage{algorithmic}
\usepackage[ruled,vlined]{algorithm2e}
\usepackage{hyperref}

\hypersetup{colorlinks,linkcolor = violet, urlcolor  = YellowOrange, citecolor = Aquamarine, anchorcolor = YellowOrange}

\DeclareMathOperator*{\gap}{\mathsf{err}}
\DeclareMathOperator*{\pinf}{\vphantom{p}inf}

\newcommand{\real}{\ensuremath{\mathbb{R}}}
\newcommand{\En}{\ensuremath{\mathbb{E}}}

\newcommand{\ind}{\mathbb{I}}

\newcommand{\simplex}{\Delta}

\newcommand{\defn}{:\,=}

\newtheorem{theorem}{Theorem}
\newtheorem{lemma}{Lemma}

\newtheorem{corollary}{Corollary}
\newtheorem{proposition}{Proposition}

\newtheorem{assumption}{Assumption}

\newcommand{\1}{\ensuremath{{\sf (i)}}}
\newcommand{\2}{\ensuremath{{\sf (ii)}}}

\newcommand{\algdet}{Comptron\xspace}
\newcommand{\algnoise}{Rob-Comptron\xspace}

\DeclareMathOperator*{\argmin}{argmin}
\DeclareMathOperator*{\argmax}{argmax}
\DeclareMathOperator*{\sign}{sign}

\makeatletter
\long\def\@makecaption#1#2{
        \vskip 0.8ex
        \setbox\@tempboxa\hbox{\small {\bf #1:} #2}
        \parindent 1.5em  
        \dimen0=\hsize
        \advance\dimen0 by -3em
        \ifdim \wd\@tempboxa >\dimen0
                \hbox to \hsize{
                        \parindent 0em
                        \hfil
                        \parbox{\dimen0}{\def\baselinestretch{0.96}\small
                                {\bf #1.} #2
                                }
                        \hfil}
        \else \hbox to \hsize{\hfil \box\@tempboxa \hfil}
        \fi
        }
\makeatother

\newcommand{\xdist}{\mathcal{D}_\x}
\newcommand{\X}{\mathcal{X}}
\newcommand{\Y}{\mathcal{Y}}
\newcommand{\F}{\mathcal{F}}
\newcommand{\Flin}{\mathcal{\F}_{\textsf{lin}}}
\newcommand{\x}{x}
\newcommand{\y}{y}
\newcommand{\f}{f}
\newcommand{\bx}{\mathbf{\x}}
\newcommand{\by}{\mathbf{\y}}
\newcommand{\hf}{\hat{f}}
\newcommand{\tf}{\tilde{f}}
\newcommand{\hfk}{\hf_{\kor}}
\newcommand{\hfkn}{\hf_{\kor, \nsamp}}

\newcommand{\pfrob}{p_{\sf{rob}}}
\newcommand{\ferm}{\f_{\sf{ERM}}}
\newcommand{\fdim}{d}
\newcommand{\nsamp}{n}
\newcommand{\xset}{S}

\renewcommand{\u}{u}
\newcommand{\tu}{\tilde{u}}
\newcommand{\hu}{\hat{u}}
\newcommand{\utrue}{u^*}
\newcommand{\uclass}{\mathcal{U}}
\newcommand{\util}{U}
\newcommand{\ugap}{\u_{\mathsf{gap}}}
\newcommand{\hugap}{\hat{\u}_{\sf gap}}
\newcommand{\utrgap}{\utrue_{\sf gap}}
\newcommand{\hutil}{\hat{\util}}

\newcommand{\kor}{k}
\newcommand{\oracle}{\mathcal{O}}
\newcommand{\prefto}{\succeq}

\newcommand{\query}{q}
\newcommand{\qry}{\query}
\newcommand{\nq}{\eta}

\newcommand{\nqry}{n_\qry}

\newcommand{\ld}{p}
\newcommand{\imax}{i_{\sf{max}}}
\newcommand{\conf}{\delta}

\newcommand{\minmax}{\mathfrak{M}}
\newcommand{\rad}{\mathfrak{R}}
\newcommand{\comm}{\circ}

%% file: abs.tex
\begin{abstract}
Traditional learning approaches for classification implicitly assume that each mistake has the same cost. 
In many real-world problems though, the utility of a decision depends on the underlying context $\x$ and decision $\y$; for instance, misclassifying a stop sign is worse than misclassifying a road-side postbox. However, directly incorporating these utilities into the learning objective is often infeasible since these can be quite complex and difficult for humans to specify.

We formally study this as \emph{agnostic learning with unknown utilities}: given a dataset $\xset = \{\x_1, \ldots, \x_\nsamp\}$ where each data point $\x_i \sim \xdist$ from some unknown distribution $\xdist$, the objective of the learner is to output a function $\f$ in some class of decision functions $\F$ with small excess risk. This risk measures the performance of the output predictor $\f$ with respect to the best predictor in the class $\F$ on the \emph{unknown} underlying utility  $\utrue:\X\times\Y\mapsto [0,1]$. This utility $\utrue$ is not assumed to have any specific structure and is allowed to be any bounded function. This raises an interesting question whether learning is even possible in our setup, given that obtaining a generalizable estimate of utility $\utrue$ might not be possible from finitely many samples. Surprisingly, we show that estimating the utilities of only the sampled points~$\xset$ suffices to learn a decision function which generalizes well. 

With this insight, we study mechanisms for eliciting information from human experts which allow a learner to estimate the utilities $\utrue$ on the set $\xset$.
While humans find it difficult to directly  provide  utility values reliably, it is often easier for them to provide comparison feedback based on these utilities. We show that, unlike in the realizable setup, the vanilla comparison queries where humans compare a pair of decisions for a single input $\x$ are insufficient. We introduce a family of elicitation mechanisms by generalizing comparisons, called the $\kor$-comparison oracle, which enables the learner to ask for comparisons across $\kor$ different inputs $x$ at once. We show that the excess risk in our agnostic learning framework decreases at a rate of $O\left(\frac{1}{\kor} \right)$ with such queries. This result brings out an interesting accuracy-elicitation trade-off -- as the order $\kor$ of the oracle increases, the comparative queries become harder to elicit from humans but allow for more accurate learning.
\end{abstract}

%% file: intro.tex
\section{Introduction}\label{sec:intro}
Our focus is on learning predictive models for decision-making tasks. Current paradigms for classification tasks use datasets consisting of scenarios\footnote{We use the term scenario/context/feature for the vector $\x$ interchangeably.} $\x$  along with the decisions $\y$ taken by human experts to learn a decision function\footnote{We consider finite decision spaces $\Y$.} $\f:\X \mapsto \Y$. For instance, in economics such decisions correspond to whether buyers bought an item at a suggested price~\cite{afriat1967,beigman2006}, in robotics such feedback comprises expert demonstrations in imitation learning~\cite{abbeel2004, argall2009}, and in machine learning literature such supervision consists of labels selected by human annotators~\cite{bishop2006, duda2012}.

When we optimize models to predict correctly on these datasets, we often implicitly assume that all mistakes are equally costly, and that each scenario $\x$ in the data is just as important. In reality though, this is rarely the case. For instance, the standard $0-1$ loss for classification tasks assigns a unit of loss for each mistake, but misclassifying a stop sign is significantly more dangerous than misclassifying a road-side postbox. In Figure~\ref{fig:intro}, we expand on this insight and illustrate how learning from such revealed decisions can often
lead to suboptimal decision functions.

What is missing from this classical framework is that for most decision-making tasks there exists an underlying function \mbox{$\utrue:\X\times \Y \mapsto [0,1]$} which evaluates the utility of a decision $\y$ depending on the surrounding context~$\x$. Depending on the decision task, such utility functions can encode buyer preferences in economics, rewards for robotic skills, or misprediction costs for classification. However, these utility functions are a priori unknown to the learner since the dataset consists only of context-decision pairs $(\x,\y)$. Furthermore, asking human experts to write down these complex utility functions 
can be quite challenging and prone to serious errors~\cite{amodei2016}.

One commonly studied approach, referred to as learning from revealed preferences in economics~\cite{beigman2006, balcan2014} and inverse reinforcement learning (IRL) in the machine learning literature~\cite{ng2000, ziebart2008}, assumes that the utility function $\utrue$ belongs to some pre-specified class and uses the fact that decision $\y$ was the optimal decision for scenario $\x$ to learn estimates of these utilities. This setup is called the well-specified or realizable setup. However, this posited utility class can be misspecified in that the underlying utility $\utrue$ might not belong to this class. The correctness of such learning approaches crucially relies on the well specified assumption and offers no guarantees on how their performance degrades in the presence of class misspecifications.

\begin{figure}[t!]
  \centering\hspace*{-4ex}
  \captionsetup{font=small}

\begin{tabular}{c}
  \vspace{-1mm}\includegraphics[width=0.75\textwidth]{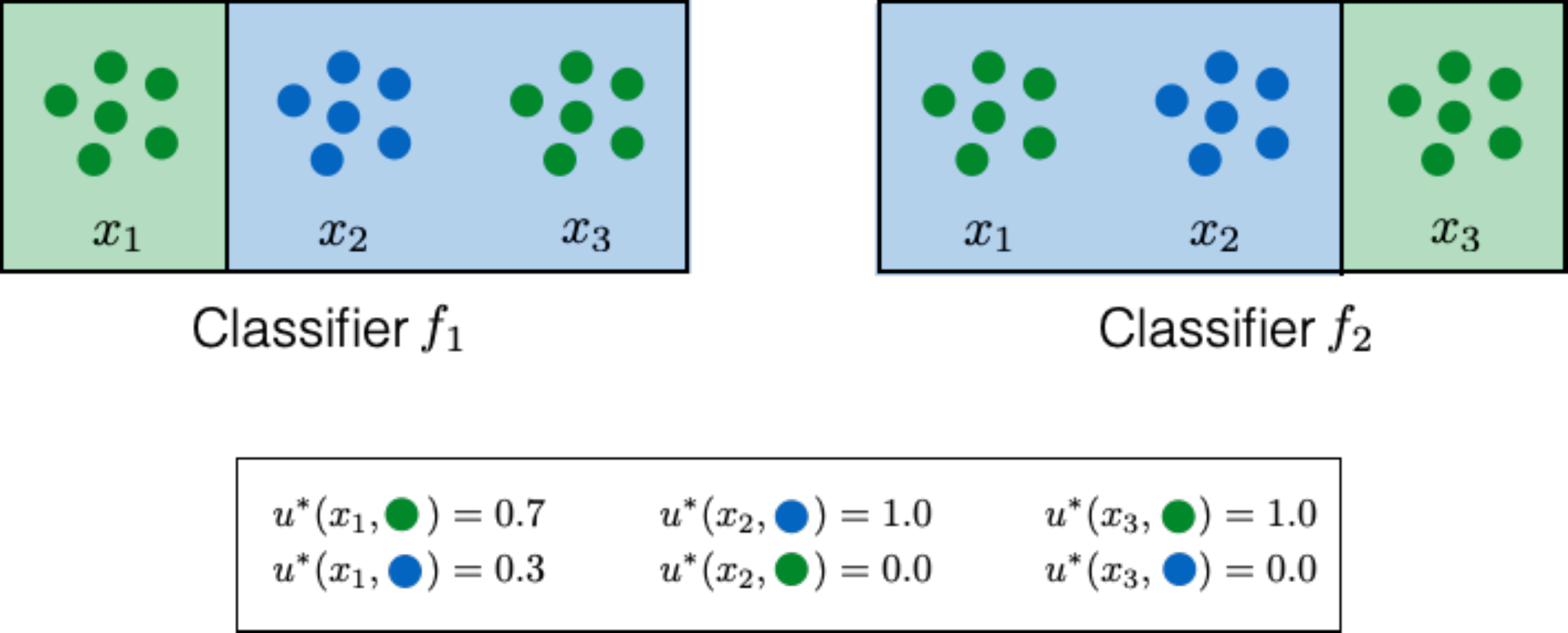}
  \vspace{3mm}
\end{tabular}
\caption{\small Consider a binary decision-task with decisions G(reen) and B(lue). The instance space comprises of three  equiprobable clusters of datapoints $\x_1, \x_2$ and $\x_3$, and have associated utilities $\utrue$ for decisions B and G. The colour of the datapoints represents the decision with higher utility. The function class $\F$ consists of linear predictors. In the traditional learning setups where the dataset consists of pairs $(\x, \y)$, no learner will have enough information to select between $\f_1$ and $\f_2$ since the $0-1$ error for both is $\nicefrac{1}{3}$. In contrast, using a $2$-comparison oracle, a learner can ask a query of the form ``Which of $\utrue(\x_1, \text{G}) +  \utrue(\x_3, \text{B})$ or $\utrue(\x_1, \text{B}) +  \utrue(\x_3, \text{G})$ is bigger?". This allows them to infer that correctly predicting $\x_3$ gives a higher overall utility and output the optimal decision function~$\f_2$. \vspace{-3mm}}
  \label{fig:intro}
\end{figure}

We overcome this uncertainty in specifying the utility function $\utrue$ by proposing an \emph{agnostic learning} framework which places no assumptions on the class of utility functions. Instead, we consider decision functions belonging to some class $\F = \{ \f \;|\; \f:\X \mapsto \Y\}$ and study the objective of obtaining the ``best" decision rule in $\F$ with respect to the unknown utility $\utrue$.  Formally, given the decision class $\F$ and samples from a distribution $\xdist$ over the feature space $\X$, the objective of the learner is to output a model $\hf \in \F$ with small excess risk or regret
\begin{align}
  \gap(\hf, \F) \defn \sup_{\f\in \F} \En_{\x \sim \xdist}[\utrue(x, \f(x))] -  \En_{\x \sim \xdist}[\utrue(x, \hf(x))]\;.
\end{align}
Our proposed notion of excess risk measures the performance of an estimator $\hf$ by comparing its decisions with those of the best predictive model in the class $\F$ under the utility $\utrue$. Contrast this with the classical agnostic learning framework~\cite{haussler1992} where the evaluation metric for classification measures what proportion of datapoints $\hf$ predicts correctly
\begin{align}
  {\sf{err}}_{\sf{cl}}(\hf, \F)\defn \sup_{\f\in \F} \En_{\x \sim \xdist}[\ind[\f(x) \neq \y_\x]] -  \En_{\x \sim \xdist}[\ind[\hf(x) \neq \y_\x]]\;,
\end{align}
where $\y_x = \argmax_{y \in \Y} \utrue(\x, \y)$ represents the expert decision (revealed decision) for scenario $\x$. Our above framework generalizes the proper agnostic learning framework -- we restrict our attention to proper learners which output models $\hf \in \F$ and the decision class $\F$ is agnostic towards the unknown underlying utility $\utrue$. Indeed, our agnostic framework allows for misspecification in the decision class $\F$ and allows for situations where no predictive model $\f \in \F$ matches the expert predictions $\y_x$ for all instances $\x$.

As highlighted by Figure~\ref{fig:intro}, such a misspecification in the function class $\F$ implies that no decision function $\f \in \F$ will be able to perfectly fit these optimal decisions $\y_x$ for all points $\x\in \xset$. In order to solve the agnostic learning problem, it is necessary for the learner to understand the how costly these different mistakes are relative to each other. From the learners perspective, observing only the optimal decisions $\y_x$  for each instance $\x$, such as revealed preferences or expert demonstrations, are clearly insufficient to obtain any information about these costs.
One way to overcome this information-theoretic limit of revealed decisions is to directly elicit the utilities from humans -- for scenarios $\x$ and decision $\y$, ask an expert ``What is the utility $\utrue(\x, \y)$ for taking the decision $\y$ given situation $\x$?". However, since the underlying utility $\utrue$ can be quite complex, humans are inept at answering them reliably~\cite{miller1956, stewart2005}. For instance, it can be challenging for humans to correctly specify the costs of mispredicting, say, a stop sign as a red signal relative to that of predicting it as a post-box.

On the other hand, it is often easier for humans to provide comparative evaluations based on these utilities~\cite{thurstone1927, furnkranz2010} and allow the learner to obtain relative feedback. Using these, the learner can query an expert with comparison or preference queries asking ``For instance $\x$, which of the two utilities $\utrue(\x, \y_1)$ or $\utrue(\x, \y_2)$ is larger?". Such vanilla comparisons can allow the learner to infer relative utilities for decisions $\y_1$ and $\y_2$ for a given context $\x$; the learner can conclude that mispredicting stop sign as post-box is worse than mispredicting it as a red signal. However, such feedback still does not provide any information about the mistake costs across different examples -- given a choice, should the learner correctly predict a stop-sign or correctly predict a post-box?

While vanilla comparisons are insufficient for the agnostic setup, let us consider the other extreme: suppose that we have access to an oracle which can provide us with comparisons of overall utilities for functions $\f_1, \f_2 \in \F$. That is, the oracle can answer question of the form ``Which of the two overall utilities $\En_{x}[\utrue(x, \f_1(x))]$ or $\En_{x}[\utrue(x, \f_2(x))]$ is larger?". Given access to such an oracle, we will be able to find the optimal classifier in the class $\F$. We call this the $\infty$-comparison oracle since such preferences requires a human to reason about the utilities over the entire feature space $\X$ at once. Even for a small image classification task with a million images, this would require a human to compare the utility of a million simultaneous predictions! While this approach does allows for optimal estimation, the trade-off is that it puts the complete burden of learning on the human's side. It is worth highlighting that the comparisons between lotteries used to establish the von Neumann-Morgenstern utility theorem~\cite{morgenstern1953} can be shown to be a special case of such an $\infty$-comparison oracle.

While comparison queries only allow comparison within a single instance, the $\infty$-comparison oracle takes the other extreme and requires a comparison along all instances. However, we need not restrict our self to either of these extremes; our key insight is that there is a natural spectrum of such comparisons, which we call $\kor$-comparisons which interpolate between the single or $1$-comparison and the $\infty$-comparison oracle. Such comparison queries allow a learner to pick $\kor$ instances $\{\x_1, \ldots, \x_k\}$ and two sets of corresponding decision, $\{\y_1, \ldots, \y_\kor\}$ and $\{\y'_1, \ldots, \y'_\kor\}$, and ask ``Which of the cumulative utilities $\sum_i \utrue(\x_i, \y_i)$ or
$\sum_i \utrue(\x_i, \y'_i)$ is bigger?". For instance, for the example in Figure~\ref{fig:intro}, giving the learner access to a $2$-comparison oracle allows the algorithm to output the optimal decision function.

These higher-order comparison oracles form a natural hierarchy of elicitation mechanisms for the learner with a $\kor'$-oracle being strictly more informative than a $\kor$-oracle for $\kor' > \kor$. They allow for a natural trade-off between accuracy and elicitation in the learning with unknown utilities framework. As we increase the order $\kor$ of the oracle, the learner can obtain finer information about the utilities $\utrue$ and output functions with lower excess risk. However, this increase in information comes at the expense of asking for a harder elicitation from the human expert.

\paragraph{Our Contributions.} We propose a novel framework, which we call \emph{agnostic learning with unknown utilities}, for studying decision problems wherein the learner is evaluated with respect to an unknown utility function. Within this framework, we show that standard approaches which work well in the realizable setup, such as revealed preferences as well as vanilla comparisons, can perform quite poorly in the face of misspecification and can have excess risk $\Omega(1)$. To overcome this, we propose a family of elicitation mechanisms, the $\kor$-comparisons, which allows the learner access to finer information from an human expert with increasing values of the order $\kor$. Our main results, detailed in Section~\ref{sec:main-res}, provide a tight characterization of the excess risk as a function of the order $\kor$ of the comparison oracle available to the learner. These result brings out an interesting accuracy-elicitation trade-off -- as the order $\kor$ of the oracle increases, the comparative queries allow for more accurate learning in our setup but become harder to elicit from humans.

We would like to highlight that increasing the order $\kor$ of the comparisons could lead to potentially biased and noisy responses from the human expert. As a consequence, there might be an additional trade-off involving the \emph{quality of the information} obtained by increasing the order. While we do not focus on this aspect of elicitation, it is an interesting direction for future work.

\paragraph{Paper Organization.} The remainder of the paper is organized as follows: Section~\ref{sec:prob-form}  introduces our agnostic learning with unknown utilities problem setup and the $\kor$-comparison elicitation mechanism, and Section~\ref{sec:main-res} gives an overview of our main results and algorithmic contributions. In Section~\ref{sec:bin-pred}, we study excess risk bounds for the
binary decision problem in our framework and propose our algorithm, \algdet, to learn from higher-order comparisons and in Section~\ref{sec:local-minimax}, we study adaptive estimators which are optimal for each instance of our problem.

%% file: prob.tex
\section{Problem formulation}\label{sec:prob-form}
In this section, we formally state our learning with unknown utilities problem and introduce the $\kor$-comparison oracle. Let $\X \subseteq \real^\fdim$ represent the space of feature vectors, $\Y$ denote the corresponding decision space and $\F$ denote a class of decision making functions, given as $\F = \{\f \; | \; \f:\X \mapsto \Y \}$. Our framework considers an underlying utility function $\utrue: \X\times \Y \mapsto [0,1]$ which assigns a non-negative real value for making a decision $\y \in \Y$ given a situation $\x \in \X$. Further, let us denote the set
\begin{align}\label{eq:uclass}
\uclass = \{ \u\; |\; \u: \X \times \Y \mapsto [0,1] \}
\end{align}
of all possible such utility functions. For any distribution $\xdist$ over the feature space $\X$, we define the expected utility of a decision function $\f \in \F$ as $\util(\f;\utrue) \defn \En_{\x \sim \xdist}[ \utrue(\x, \f(\x))]$.
Observe that such an expected utility model assumes that the utilities are additive across the different instances $\x$ and is a commonly studied model both in the machine learning, statistics and economics literature. We denote the excess risk of a function $\f$ with respect to the function class $\F$ by
\begin{align}\label{eq:excess-risk}
  \gap(\f, \F; \utrue) \defn \max_{\f'\in \F}\util(\f';\utrue) - \util(\f;\utrue).
\end{align}
Further, we denote the optimal decision for any instance $\x$ with respect to the underlying utility $\utrue$ by  \mbox{$\y_\x \defn \argmax_{\y \in \Y} \utrue(\x, \y)$}.

Similar to the classical agnostic learning setup~\cite{haussler1992}, we assume that the learner does not know the underlying distribution $\xdist$ of the instances. However, our setup differs from it in that we do not assume that the underlying utility function $\utrue$ is known to the learner. Instead, we provide the learner access to an oracle which allows the learner to elicit responses to higher-order preferences queries.

\paragraph*{Comparison Oracle} Since the utility function $\utrue$ is unknown to the learner, our framework allows the learner access to an oracle which provides comparative feedback based on the utilities $\utrue$. We consider a family of such oracles $\oracle_\kor$, each indexed by its order $\kor$ which determines the number of different instances the learner is allowed to specify in the comparison query. For an oracle $\oracle_\kor$, a learner is allowed to select a set of $\kor$ situations $\bx \in \X^\kor$ and two pairs of corresponding decisions $\by_1, \by_2 \in \Y^\kor$. The oracle then compares, in a possibly noisy manner, the cumulative utilities of the pair $(\bx, \by_1)$ and $(\bx, \by_2)$ and responds with the feedback on which one is larger. As the order $\kor$ of the oracle increases, the queries become more complex -- an expert is required to evaluate a larger number of instances at once. This family of comparison oracles captures a natural hierarchy of elicitation mechanisms where with each increasing value of $\kor$, a learner has access to more information about the utility function $\utrue$.

Formally, we represent a $\kor$-query by a tuple $(\bx, \by_1, \by_2)$ where the input $\bx = (\x_1, \ldots, \x_\kor)$ comprises $\kor$ feature vectors and the corresponding decision vectors $\by_1 = (\y_1, \ldots, \y_\kor)$ and $\by_2 = (\y'_1, \ldots, \y'_\kor)$.\footnote{We overload our notation and represent the cumulative utilities of the $\kor$ inputs $(\bx, \by)$ by  \mbox{$\utrue(\bx, \by) = \sum_{i}\utrue(\x_i, \y_i)$}.} Given such a query $\qry$, the oracle $\oracle_\kor$ provides the learner a binary response
\begin{align}\label{eq:oracle-def}
  \oracle_\kor(\query = (\bx, \by_1, \by_2)) = \begin{cases}
  \ind\left[\utrue(\bx, \by_1) \geq \utrue(\bx, \by_2)  \right] \quad &\text{with prob. } 1 - \nq_\qry\\
  1 - \ind\left[\utrue(\bx, \by_1) \geq \utrue(\bx, \by_2)  \right] \quad &\text{otherwise}
\end{cases},
\end{align}
where the parameter $0 \leq \nq_\qry < \frac{1}{2}$ represents the noise level corresponding to query $\qry$. Thus, the oracle\footnote{Note that while the oracle depends on the underlying utility function $\utrue$, our notation suppresses this dependence for clarity. We use the notation $\oracle_\kor(\qry;\utrue)$ whenever we want to make this dependence explicit.} $\oracle_\kor$ provides noisy comparisons of the cumulative utilities $\utrue(\bx, \by_1)$ and $\utrue(\bx, \by_2)$ with varying noise level $\nq_\qry$. Observe that we allow the noise levels $\nq_\qry$ to be different for each query $\qry$. 

\paragraph*{Problem Statement} We are interested in the \emph{agnostic learning with unknown utilities} problem where a learner is provided~$\nsamp$ samples $\xset = \{\x_1, \ldots, \x_\nsamp\}$ with each $\x_i \sim \xdist$ and access to the  $\kor$-comparison oracle described above, and is required to output a decision function $\hf \in \F$ such that error $\gap(\hf, \F)$ is small. The caveat is to do so with a minimum number of calls, which we term the query complexity $\nqry$ of learning, to the comparison oracle $\oracle_\kor$. Quantitatively, we would like to characterize the excess risk from equation~\eqref{eq:excess-risk} in terms of the number of sampled instances~$\nsamp$,  the order~$\kor$ of the comparison oracle and properties of the decision function class~$\F$, and the associated oracle query complexity~$\nqry$ to obtain this bound.

Obtaining such bounds on the excess risk $\gap(\f, \F;\utrue)$ in terms of the order $\kor$ allow us to quantify the trade-offs in learning better decision functions at the expense of requiring more complex information from the human expert. Going forward, we focus on the binary decision making problem where the label space $\Y = \{0,1\}$ for clarity of exposition. Whenever our results can be extended to arbitrary decision sets, we provide a small remark about this extension.

%% file: results.tex
\section{Main results}\label{sec:main-res}
With the formal problem setup in place, we discuss our main results for learning in this framework of unknown utilities. At a high level, our objective is to understand how the excess risk $\gap(\f, \F;\utrue)$ defined in equation~\eqref{eq:excess-risk} behaves as a function of the oracle order $\kor$ -- specifically, at what rates does learning in our proposed framework get easier as we allow learner to elicit more complex information from the oracle?

For our main results, on the upper bound side, we design estimators for learning from the $\kor$-comparison oracle, and on the lower bound side, we study information-theoretic limits of learning with such higher-order comparisons. While we state our results for the binary decision problem where the label space $\Y = \{0,1\}$ for clarity, most of our results can be generalized to arbitrary outcome space $\Y$.

\subsection{Excess risk with $\kor$-comparison oracle (Section~\ref{sec:bin-pred})} We study a class of \emph{plug-in} estimators which are based on the following two-step procedure:
\begin{itemize}
  \item[i.] Obtain estimate $\hu$ of the true utility $\utrue$ on the sampled datapoints.\vspace{-1mm}
  \item[ii.] Output utility maximizing function $\hfkn$ with respect to the estimated utility $\hu$.
\end{itemize}

For learning the parameters $\hu$, we introduce the \algdet (Algorithm~\ref{alg:det-bin}) and \algnoise (Algorithm~\ref{alg:stoch-bin}) algorithms for the noiseless and noisy comparison oracles respectively. We show that when these estimates $\hu$ are combined with the two-step plug-in estimator, the excess risk of the function $\hfkn$ scales as $O(\frac{1}{\kor})$ and an additive complexity term capturing uniform convergence of the decision class $\F$ with respect to the true utility $\utrue$.
\begin{theorem}[Informal, noiseless comparisons]
  Given $\nsamp$ samples, the excess risk for the function $\hfkn \in \F$ output by the plug-in estimator using estimates $\hu$ from \algdet satisfies
  \begin{align*}
    \gap(\hfkn, \F; \utrue) \leq {\sf{Complexity}}_\nsamp(\F;\utrue) + O\left( \frac{1}{\kor}\right)\cdot \left(\frac{1}{\nsamp}\sum_{i=1}^\nsamp\ind[\ferm(\x_i) \neq \y_i] \right)\;,
  \end{align*}
  where the ERM function $\ferm \in \argmax_{f \in \F} \sum_{i=1}^n \utrue(x_i, f(x_i))$. Furthermore, \algdet makes only $O(\nsamp\log \kor)$ queries to the oracle $\oracle_\kor$.
\end{theorem}
We make a few remarks on this result. First, observe that the complexity term depends on the true utility function $\utrue$ and not on the estimates $\hu$. This ensures that the complexity term does not depend on the utility class $\uclass$ but rather only on the specific utility $\utrue$ -- indeed, the class $\uclass$ consists of all bounded function and uniform convergence might not even be possible with finite sample for a large class of distributions $\xdist$. Second, the additional error of $O(\frac{1}{\kor})$ accounts for the fact that the utilities $\utrue$ are unknown. One can learn better decision functions by increasing the order $\kor$ of the comparison oracle but this comes at the cost of the human expert answering a more complex set of queries. Furthermore, this error is multiplied by the $0-1$ prediction error of the optimal on-sample classifier $\ferm = \argmax_{\f\in \F}\sum_i\utrue(\x_i, \f(\x_i))$. This implies that in the well-specified setup, where there exists an $\f\in\F$ such that $\f(\x_i) = \y_i$ on the sampled datapoints, the second term becomes $0$ and the learner pays no additional error for not knowing the utilities $\utrue$. Third, observe that our proposed algorithms, \algdet and \algnoise, are query efficient; both require only $O(\nsamp \log \kor)$ calls to the $\kor$-comparison oracle to produce ``good" estimates $\hu$.

The proof of the above theorem proceeds in two steps. First, we adapt the classical proof for upper bounding the risk of ERM procedures to show that the gap $\gap(\hfkn, \F)$ decomposes into the complexity term and estimation error $\|\hu - \utrue \|_{\xset,\infty}$, evaluated on the dataset $\xset$. Next, we show that this estimation error scales as $O\left(\frac{1}{\kor}\right)$ for the \algdet and \algnoise procedures.

Next, we address the optimality of the above plug-in procedure by studying the information-theoretic limits of learning with a $\kor$-comparison oracle. Specifically, in Theorem~\ref{thm:bin-pred-lower} we establish that the rate of $\frac{1}{\kor}$ is indeed minimax optimal -- for any $\kor > 1$ and any predictor $\hf$ in some class $\F$, we can construct utility functions $\utrue$ such that excess risk $\gap(\hf, \F;\utrue) = \Omega\left(\frac{1}{\kor} \right)$. These lower bounds imply that traditional comparison based learning, corresponding to $\kor = 1$, is insufficient for learning good decision rules in our framework.

\subsection{Instance-optimal learning (Section~\ref{sec:local-minimax}).} While the previous results show that the error rate of $O(\frac{1}{\kor})$ is optimal on worst-case instances, some instances of our learning with unknown utilities problem might be easier than these worst-case ones and one would expect the excess risk to be smaller for them. In this section, we study estimators whose error adapts to hardness of the specific problem instance.

To begin with, in Proposition~\ref{prop:subopt-plugin} we establish that the plug-in estimator with \algdet estimates $\hu$ is not optimal for all instances -- it does not adapt to these easier instances. Inspired from the robust optimization literature, we introduce a randomized estimator $\pfrob$ and show that it is instance-optimal. Informally, we establish in Theorem~\ref{thm:local-upper} that for any instance $(\xdist, \utrue, \F)$ of the problem, the excess risk for $\pfrob$ is characterized by a local modulus of continuity; this modulus captures how quickly the optimal decision function in class $\F$ can change in a small neighborhood around $\utrue$ for the distribution $\xdist$. In Theorem~\ref{thm:local-lower}, we derive a lower bound on the \emph{local minimax excess risk} and show that the local modulus is indeed the correct instance-dependent complexity measure for this problem.

However, note that such adaptivity to the hardness of the instance comes at the cost of query efficiency. Our estimator $\pfrob$ makes an exponential number $O(\nsamp^\kor)$ of calls to the oracle $\oracle_\kor$.

%% file: bin-class.tex
\section{Binary decision-making with $\kor$-comparisons}\label{sec:bin-pred}
In this section, we obtain upper and lower bounds on the excess risk for the binary prediction problem with unknown utilities where the learner can elicit utility information using a $\kor$-comparison oracle. In Section~\ref{sec:bin-upper}, we introduce algorithms which learn decision-making rules from higher-order preference queries and obtain upper bounds on the excess risk for such estimators. Then, in Section~\ref{sec:bin-lower}, we turn to the information-theoretic limits of learning from $\kor$-queries and obtain lower bounds on the minimax risk of any estimator.

Recall from Section~\ref{sec:prob-form}, our setup gives the learner access to a dataset $\xset = \{\x_1, \ldots, \x_\nsamp\}$ comprising $\nsamp$ points, each sampled i.i.d. from an underlying distribution $\xdist$ and to a comparison oracle $\oracle_\kor$. Before proceeding to define the estimator, we introduce some notation. For any function $\f\in \F$, let us denote the empirical cumulative utility with respect to utility function~$\utrue$ and the corresponding empirical utility maximizer as
\begin{align}\label{eq:ferm}
  \hutil_\nsamp(\f;\utrue)=\frac{1}{n}\sum_{i}\utrue(\x_i, \f(\x_i))\quad \text{and} \quad \ferm \in \argmax_{\f \in \F}\hutil_\nsamp(\f;\utrue)\;,
\end{align}
where the subscript $\nsamp$ encodes the dependence on the number of samples. If the underlying utility $\utrue$ were in fact known to the learner, it would have output the classifier $\ferm$, which, from the classical learning theory literature, is known to have favorable generalization properties~\cite{shalev2010}. For the case of unknown utilities, we extend this ERM procedure to a natural two-stage plug-in estimator which outputs the minimizer with respect to an estimate $\hu_\kor$ of these utilities.

\subsection{Excess-risk upper bounds for plug-in estimator}\label{sec:bin-upper}
Building on the ERM estimator $\ferm$ described in equation~\eqref{eq:ferm}, we design a two stage \emph{plug-in} estimator $\hf_{\kor, \nsamp}$, where the subscript $\kor$ represents the order of the comparison oracle used to obtain the estimate.

In the first stage, we form estimates $\hu_\kor$ of the true utility function $\utrue$ on the sampled datapoints $\xset$ using the $\kor$-comparison oracle. The predictor $\hfkn \in \F$ is then given by the empirical utility maximizer with respect to $\hu_\kor$, that is,
\begin{align}\label{eq:plug-in}
\hfkn \in \argmax_{\f\in \F} \frac{1}{n} \sum_{i=1}^n \hu_\kor(\x_i, \f(\x_i)).
\end{align}
Before detailing out the procedures for producing utility estimates $\hu_\kor$, we present our first main result which shows that the excess risk $\gap(\hfkn, \F;\utrue)$ can be upper bounded as a sum of two terms: (i) a complexity term corresponding to the rate of uniform convergence of the cumulative utility $\util(\f;\utrue)$ over the decision class $\F$ and (ii) an estimation error term which denotes how well the estimates $\hu_\kor$ approximate $\utrue$ on the sampled datapoints. Our result measures this estimation error in terms of a data-dependent norm
\begin{align}\label{eq:norm-est-error}
  \|\u\|_{\xset, \infty} \defn \sup_{i \in [\nsamp]} \sup_{\y \in \Y} |\u(\x_i, \y)|.
\end{align}
Recall from equation~\eqref{eq:ferm} that the function $\ferm$ is the minimizer of the empirical utility $\hutil_\nsamp(\f;\utrue)$. While the following results hold for general decision spaces $\Y$, we later specialize this in Proposition~\ref{prop:upper_bin} for the binary prediction setup.  

\begin{theorem}[Excess-risk upper bound] \label{thm:upper_finite} Given datapoints $\xset = \{\x_1, \ldots, \x_\nsamp \}$ such that each $\x_i \sim \xdist$, and an estimate $\hu_\kor$ of the true utility function $\utrue$, the plug-in estimate $\hfkn$ from equation~\eqref{eq:plug-in} satisfies
  {\small
  \begin{align}\label{eq:thm_upper_finite}
    \gap(\hfkn, \F; \utrue) \leq 2\cdot\sup_{\f \in \F}\left(|\util(\f;\utrue) - \hutil_\nsamp(\f;\utrue)|\right) + 2\|\utrue - \hu_\kor\|_{\xset, \infty}\cdot \left(\frac{1}{\nsamp}\sum_{i=1}^\nsamp \ind[\ferm(\x_i) \neq \hfkn(\x_i)] \right).
  \end{align}}
\end{theorem}

A few comments on Theorem~\ref{thm:upper_finite} are in order. First, notice that the upper bound on the risk $\gap(\hfkn, \F;\utrue)$ is a deterministic bound comprising two terms. The uniform convergence term captures how fast the empirical utility $\hutil_\nsamp(\f;\utrue)$ converge to the population utility $\util(\f;\utrue)$ uniformly over the decision class $\F$. Using standard bounds~\cite{bartlett2002}, one can show that this term is upper bounded by the empirical Rademacher complexity of the class $\F$ on the datapoints $\xset$, that is,
\begin{align}\label{eq:unif-rademacher}
\sup_{\f \in \F}\left(|\util(\f;\utrue) - \hutil_\nsamp(\f;\utrue)|\right) \leq \En_{\varepsilon} \left[\sup_{\f \in \F} \left\lvert\frac{1}{\nsamp} \sum_{i=1}^\nsamp \varepsilon_i \utrue(\x_i, \f(\x_i))\right\rvert \right] :\;= \widehat{\rad}_\nsamp(\F\comm\utrue)\;
\end{align}
where each $\varepsilon_i$ is an i.i.d. Rademacher random variable taking values $\{-1, +1\}$ equiprobably. Such complexity measures are commonly studied in the learning theory literature and one can obtain sample complexity rates for a wide range of decision classes including parametric decision classes and non-parametric kernel classes amongst others.

The second term in equation~\eqref{eq:thm_upper_finite} is given by a product of two terms. The first part $\|\utrue - \hu_\kor\|_{\xset, \infty}$ captures the \emph{on-sample approximation error} of the estimates $\hu_\kor$. Notice that, in general, the problem of estimating $\utrue$ uniformly over the space $\X$ is infeasible since the class $\uclass$ contains the set of all bounded functions on $\X\times \Y$. However, the fact that we are required to estimate the utilities $\utrue$ only on the sampled datapoints $\xset$ makes learning feasible in our framework. The second part, $\frac{1}{\nsamp}\sum_{i=1}^\nsamp \ind[\ferm(\x_i) \neq \hfkn(\x_i)]\leq 1$
the mismatch between the predictions of $\ferm$, obtained with complete knowledge of $\utrue$, and of $\hfkn$, obtained from estimates $\hu_\kor$. Notice that whenever the function class $\F$ is correctly specified on $\xset$, that is, there exists a function $\f \in \F$ such that $\f(\x_i) = \y_i)$, then the predictions of $\hfkn$ and $\ferm$ will coincide. This follows since the labels $\y_i$ can be inferred using a $1$-comparison. In such a well-specified setup, this second term vanishes and we recover the upper bound in terms of the uniform convergence term. Surprisingly, this exhibits that not knowing the utility $\utrue$ affects learnability only when the function class $\F$ is misspecified.

\begin{proof}
  We begin by decomposing the excess error $\gap(\hfkn, \F; \utrue)$ and then handle each term in the decomposition separately. Recall that the function $\ferm$ is the maximizer of the empirical utility~$\hutil_\nsamp(\f;\utrue)$. Then, for any decision function $\f \in \F$, consider the error
\begin{align}\label{eq:error-decomp-upper}
    \gap(\hfkn, \f; \utrue) &= \util(\f;\utrue) - \hutil_\nsamp(\f;\utrue) + \hutil_\nsamp(\f;\utrue) - \hutil_\nsamp(\ferm;\utrue) + \hutil_\nsamp(\ferm;\utrue) - \hutil_\nsamp(\hfkn;\utrue) \nonumber\\
    &\quad + \hutil_\nsamp(\hfkn;\utrue) - \util(\hfkn;\utrue)\nonumber\\
    &\stackrel{\1}{\leq} 2\sup_{\f \in \F}\left(|\util(\f;\utrue) - \hutil_\nsamp(\f;\utrue)|\right) + \underbrace{\hutil_\nsamp(\ferm;\utrue) - \hutil_\nsamp(\hfkn;\utrue)}_{\text{Term (I)}},
  \end{align}
  where the inequality $\1$ follows by noting that $\ferm$ is the maximizer of $\hutil_\nsamp(\f;\utrue)$. We now focus our attention on Term (I) in the above expression.
\begin{align*}
  \hutil_\nsamp(\ferm;\utrue) - \hutil_\nsamp(\hfkn;\utrue) &= \hutil_\nsamp(\ferm;\utrue) - \hutil_\nsamp(\ferm;\hu) + \hutil_\nsamp(\ferm;\hu) - \hutil_\nsamp(\hfkn;\hu) \\
  &\quad + \hutil_\nsamp(\hfkn;\hu) - \hutil_\nsamp(\hfkn;\utrue)\\
  &\stackrel{\1}{\leq} \frac{2}{\nsamp}\sum_{i=1}^\nsamp \ind[\ferm(\x_i) \neq \hfkn(\x_i)]\cdot\sup_{\y \in \Y}|\utrue(\x_i, \y) - \hu(\x_i, \y)|\\
  &\leq 2\|\utrue - \hu\|_{\xset, \infty} \cdot \left(\frac{1}{\nsamp}\sum_{i=1}^\nsamp \ind[\ferm(\x_i) \neq \hfkn(\x_i)] \right)\;,
\end{align*}
where $\1$ follows by noting that $\hfkn$ maximizes the utility $\hutil_\nsamp(\f;\hu)$. Plugging the bound above in equation~\eqref{eq:error-decomp-upper} completes the proof.
\end{proof}
We now specialize the result of Theorem~\ref{thm:upper_finite} to the binary prediction setup where the label space $\Y = \{0, 1 \}$. Recall that for each datapoint $\x_i$, we denote the true label by $\y_i = \argmax_\y \utrue(\x_i, \y)$. We now introduce the notion of utility gaps $\ugap(\x_i)$ which measures the excess utility a learner gains by predicting a datapoint $\x_i$ correctly relative to an incorrect prediction. Formally, the gap $\ugap(\x_i)$ for datapoint $\x_i$ with respect to some utility function
$\u \in \uclass$ is given as
\begin{align}\label{eq:def-ugap}
\ugap(\x_i) \defn \u(\x_i, \y_i) - \u(\x_i, \bar{\y}_i)\;,
\end{align}
where we denote the incorrect label by $\bar{\y} = 1-\y$. With this notation, the following proposition obtains an upper bound on the excess error of plug-in estimator $\hfkn$ for the binary prediction problem in terms of the estimation error in these gaps $\ugap(\x_i)$.

\begin{proposition}[Upper bounds for binary prediction] \label{prop:upper_bin} Consider the binary decision making problem with label space $\Y = \{0, 1\}$. Given $\nsamp$ datapoints $\{\x_1, \ldots, \x_\nsamp \}$ such that each datapoint $\x_i \sim \xdist$, and an estimate $\hu_\kor$ of the utility function $\utrue$, the plug-in estimator $\hfkn$ from equation~\eqref{eq:plug-in} satisfies
  {\small
  \begin{align}
    \gap(\hfkn, \F; \utrue) \leq 2\cdot\sup_{\f \in \F}\left(|\util(\f;\utrue) - \hutil(\f;\utrue)|\right)  + 2\max_{i}[\utrgap(\x_i) - \hugap(\x_i)]\cdot \left(\frac{1}{\nsamp}\sum_{i=1}^\nsamp\ind[\ferm(\x_i) \neq \y_i] \right)\;.
  \end{align}}
\end{proposition}
The proof of the above proposition follows similar to Theorem~\ref{thm:upper_finite} and is deferred to Appendix~\ref{app:proofs-bin-finite}. This specializes the result of Theorem~\ref{thm:upper_finite} and shows that for the binary prediction problem, estimating the utility gaps $\ugap$ well for each datapoint suffices

The upper bound on excess risk given by Proposition~\ref{prop:upper_bin} shows that the function $\hfkn$ derived from estimates $\hu_\kor$ will have small error as long as the estimates $\hugap(\x_i)$ approximate the true utility gaps $\utrgap(\x_i)$ for each datapoint $\x_i$. Therefore, in the following sections, we focus on procedures for obtaining the utility estimates $\hugap$ using the $\kor$-comparison oracle. we separate the presentation based on whether the oracle $\oracle_\kor$ provides noiseless comparisons ($\nq_\qry = 0$ for all $\qry$) or whether the oracle evaluations are noisy.

\subsubsection{Estimating $\utrgap$ with noiseless oracle}
In this section, we propose our algorithm for estimating the gaps $\utrgap$ when the $\kor$-comparison oracle is noiseless. Recall from equation~\eqref{eq:oracle-def}, for a query $\query = (\bx, \by_1, \by_2)$ comprising $\kor$ feature vectors $\bx = (\x_1, \ldots, \x_{\kor})$, and two decision vectors $\by_1 = (\y_1, \ldots, \y_{\kor})$ and $\by_2 = (\y_1', \ldots, \y_{\kor}')$, such a noiseless oracle deterministically outputs
\begin{align*}
  \oracle_\kor(\query = (\bx, \by_1, \by_2)) =
  \ind\left[\utrue(\bx, \by_1) \geq \utrue(\bx, \by_2)  \right]\;,
\end{align*}
where recall that $\utrue(\bx, \by) = \sum_{i \in [k]}\utrue(\x_i, \y_i)$ is the sum of the utilities under $\utrue$ for the tuple $(\bx, \by)$. In the binary prediction setup, such queries allow a learner to specify a set of $\kor$ instances $\bx$ and a subset $\xset_\qry \subset \bx$ and ask the oracle ``whether correctly predicting instances in $\xset_\qry$ has higher utility or the instances in the complement $\bx \setminus \xset_\qry$?".

Recall that Proposition~\ref{prop:upper_bin} shows that excess risk for the plug-in estimator can be bounded by the worst-error $|\utrgap(\x_i) - \hugap(\x_i)|$ over the set of sampled datapoints $\xset$. To obtain such estimates, we introduce \emph{\algdet} in Algorithm~\ref{alg:det-bin} which is a coordinate-wise variant of the classical perceptron algorithm~\cite{rosenblatt1958}. At a high level, \algdet is an iterative procedure which estimates the utility gaps $\utrgap(\x_i)$ for each $\x_i$ relative to the largest gap
\begin{align}
\utrue_{\sf{max}} \defn \max_{i \in [\nsamp]}\utrgap(\x_i) \leq 1.
\end{align}
At each iteration $t$, the queries $\qry_{i,t}$ are selcted such that $\hugap^{t-1}(\bx, \by_1) >  \hugap^{t-1}(\bx, \by_2)$ under the current estimates $\hugap^{t-1}$. If the oracle's response is $r_{i,t}=1$, the estimates are consistent with the response and it keeps the current estimate. On the other hand, if the response $r_{i,t} = 0$, the algorithm decreases its current estimate of the $i^{th}$ datapoint in order to be consistent with this query. \algdet repeats the above procedure for $T = \log_2 \kor - 1$ timesteps and finally outputs the estimates $\hugap^{T}$.
\begin{algorithm}[t!]
	\DontPrintSemicolon
	\KwIn{Datapoints $\xset = \{\x_1, \ldots, \x_\nsamp\}$, $\kor$-comparison oracle $\oracle_\kor$}
   \textbf{Initialize:} Set $T = \log_2 k - 1 $\;
   Obtain $\y_i = \argmax_\y \utrue(\x_i, y)$ for each $i$ using $1$-comparison.\;
   Obtain index $\imax$ using $2$-comparisons such that $\imax = \argmax_i\utrgap(\x_i)$.\;
   Set initial estimates $\hugap^0 = [\hugap^0(\x_1), \ldots, \hugap^0(\x_n)] = \utrue_{\sf max}\defn \utrgap(\x_{\imax})$.\;
   (Note that exact value of $\utrue_{\sf max}$ is not required since comparison queries are relative)\;
	\For{$t = 1, \ldots, T$}{
  \For{$i = 1, \ldots, \nsamp$}{
  Denote by $\lambda = \frac{k}{2\utrue_{\sf max}}\left(\hugap^{t-1}(\x_i) - \frac{\utrue_{\sf max}}{2^t}\right)$ and query $\query_{i,t} = (\bx, \by_1, \by_2)$ where
  {\small
   \begin{equation*}
     \bx = (\underbrace{\x_{i}, \ldots, \x_{i}}_{\frac{k}{2} \text{ times }}, \underbrace{\x_{\imax}, \ldots, \x_{\imax}}_{\lambda \text{ times }}), \quad \by_1 = (\underbrace{\y_{i}, \ldots, \y_{i}}_{\frac{k}{2} \text{ times }}, \underbrace{1-\y_{\imax}, \ldots, 1- \y_{\imax}}_{\lambda \text{ times }}),\quad \by_2 = 1 - \by_1.
   \end{equation*}}\;\vspace{-6mm}
   Query oracle $\oracle_\kor$ with $\qry_{i,t}$ and receive response $r_{i,t}$.\;\vspace{1mm}
   Update $\hugap^{t}(\x_i) = \hugap^{t-1}(\x_i) - \ind[r_{i,t} = 0]\cdot\frac{\utrue_{\sf max}}{2^t}.$
   }
	}
  \KwOut{Gap estimates $\hugap^T$}
	\caption{\algdet: Comparison based Coordinate-Perceptron for estimating $\utrgap$}
  \label{alg:det-bin}
\end{algorithm}

It is worth highlighting here that \algdet initializes all the estimates as the largest gap, that is,  $\hugap^{0}(\x_i) = \utrue_{\sf max}$. Such an initialization is purely symbolic in nature and the algorithm \emph{does not} require knowledge of this value. This is because the comparison queries $\qry_{i,t}$ allows the algorithm to compare the estimates $\hugap$ with $\utrue_{\sf max}$ and the algorithm maintains its estimates $\hugap^{t}$ as a multiplicative factor of $\utrue_{\sf max}$ for iterations $t$. Further, we can use symbolic estimates to output the plug-in estimator since it is invariant to scaling the utility gaps by a  positive constant,
\begin{align}
 \argmax_{\f \in \F}\sum_{i=1}^\nsamp \hu(\x_i, \f(\x_i)) \;&\equiv\; \argmax_{\f\in \F} \sum_{i=1}^\nsamp \hugap(\x_i)\cdot \ind[\f(\x_i) = \y_i]\nonumber\\
  &\equiv\; \argmax_{\f\in \F} \sum_{i=1}^\nsamp \frac{\hugap(\x_i)}{\utrue_{\sf{max}}}\cdot \ind[\f(\x_i) = \y_i]\;\nonumber.
\end{align}

The following lemma provides an upper bound on the estimation error of \algdet and shows that the output estimates $\hugap(\x_i)$ are within a factor $O(\frac{\utrue_{\sf max}}{\kor})$ of the true gaps $\utrgap(\x_i)$.

\begin{lemma}[Estimation error of Algorithm~\ref{alg:det-bin}]\label{lem:perf-alg-det} Given access to datapoints $\xset = \{\x_1, \ldots, \x_\nsamp \}$ and $\kor$-comparison oracle $\oracle_\kor$, \algdet (Algorithm~\ref{alg:det-bin}) uses $O(\nsamp \log \kor)$ queries to the oracle and produces estimates $\hugap$ such that
  \begin{align}
    \max_{i \in [\nsamp]}\left\lvert\hugap(\x_i) - \utrgap(\x_i)\right\rvert   \leq \frac{2\utrue_{\sf max}}{\kor}\;.
  \end{align}
\end{lemma}
We defer the proof of the lemma to Appendix~\ref{app:proofs-bin-finite}. The proof proceed via an inductive argument where we show that the confidence interval around $\utrgap(\x_i)$ shrinks by a factor of $\frac{1}{2}$ in each iteration for every datapoint $\x_i$.
Given the above estimation error guarantee for \algdet, the following corollary combines these with the excess risk bounds of Proposition~\ref{prop:upper_bin} to obtain an upper bound on the excess risk of $\hfkn$.

\begin{corollary}\label{cor:bin-det}
  Consider the binary decision making problem with label space $\Y = \{0, 1\}$. Given $\nsamp$ datapoints $\{\x_1, \ldots, \x_\nsamp \}$ such that each $\x_i \sim \xdist$, the plug-in estimate $\hfkn$ from equation~\eqref{eq:plug-in}, when instantiated with the output of \algdet (Algorithm~\ref{alg:det-bin}), satisfies
    \begin{align*}
      \gap(\hfkn, \F; \utrue) \leq 2\cdot\sup_{\f \in \F}\left(|\util(\f;\utrue) - \hutil(\f;\utrue)|\right)  + \frac{2\utrue_{\sf max}}{\kor}\cdot \left(\frac{1}{\nsamp}\sum_{i=1}^\nsamp\ind[\ferm(\x_i) \neq \y_i]\right) \;.
    \end{align*}
\end{corollary}
We defer the proof of the corollary to Appendix~\ref{app:proofs-bin-finite}. Corollary~\ref{cor:bin-det} exhibits the advantage of using higher-order comparisons for the learning with unknown utilities problem -- as the order $\kor$ increases, the error of the plug-in estimate decreases additively as $O\left(\frac{1}{\kor} \right)$. It is worth noting here that while the higher-order comparisons allow the learner to better estimate the underlying utilities, the problem gets harder from the side of the human expert. Indeed, with higher values of $\kor$, the expert is required to compare utilities across $\kor$ different possible situations which can make the elicitation a harder task.

While the results in this section exhibit how the excess risk $\gap(\hfkn;\F)$ varies as a function of $\kor$, they rely on the oracle responses being noiseless. In the next section, we consider the setup where the oracle responses can be noisy and propose a robust version of the \algdet algorithm for learning in this scenario.

\subsubsection{Estimating $\utrgap$ with noisy oracle}
In contrast to the deterministic noiseless oracle of the previous section, here, we consider learning with unkown utilities when the oracle $\oracle_\kor$ can output noisy responses to each query. Recall from equation~\eqref{eq:oracle-def}, for any query $\qry$, the noisy $\kor$-comparison oracle  the correct response with probability $1-\nq_\qry$ and flips the response with probability $\nq_\qry$ for some value of $\eta_\qry < \frac{1}{2}$. While we allow this error probability to vary across different queries, we assume that this error is bounded uniformly across all queries by some constant $\nq < \frac{1}{2}$.

\begin{assumption}\label{ass:noise-bnd}
  For the noisy $\kor$-comparison oracle described in equation~\eqref{eq:oracle-def}, we have that \mbox{$\nq_\qry \leq \nq < \frac{1}{2}$} for all queries $\qry$.
\end{assumption}
\begin{algorithm}[t!]
	\DontPrintSemicolon
	\KwIn{Datapoints $\xset = \{\x_1, \ldots, \x_\nsamp\}$, $\kor$-comparison oracle $\oracle_\kor$, noise level $\nq$, confidence $\conf$}
   \textbf{Initialize:} $T = \log_2 k - 1$, $J = \frac{8}{(1-2\nq)^2}\log\left(\frac{\nsamp T}{\conf}\right)$\;
   Obtain $\y_i = \argmax_\y \utrue(\x_i, y)$ for each $i$ using $1$-comparison.\;
   Obtain index $\imax$ using $2$-comparisons such that $\imax = \argmax_i\utrgap(\x_i)$.\;
   Set initial estimates $\hugap^0 = [\hugap^0(\x_1), \ldots, \hugap^0(\x_n)] = \utrue_{\sf max}$ symbolically\;
	\For{$t = 1, \ldots, T $}{
  \For{$i = 1, \ldots, \nsamp$}{
  Denote by $\lambda = \frac{k}{2\utrue_{\sf max}}\left(\hugap^{t-1}(\x_i) - \frac{\utrue_{\sf max}}{2^t}\right)$\;
  Set query $\query_{i,t} = (\bx, \by_1, \by_2)$ where
  {\small
   \begin{equation*}
     \bx = (\underbrace{\x_{i}, \ldots, \x_{i}}_{\frac{k}{2} \text{ times }}, \underbrace{\x_{\imax}, \ldots, \x_{\imax}}_{\lambda \text{ times }}), \quad \by_1 = (\underbrace{\y_{i}, \ldots, \y_{i}}_{\frac{k}{2} \text{ times }}, \underbrace{1-\y_{\imax}, \ldots, 1- \y_{\imax}}_{\lambda \text{ times }}),\quad \by_2 = 1 - \by_1.
   \end{equation*}}\;\vspace{-6mm}
  \For{$j = 1, \ldots, J $}{
   Query oracle $\oracle_\kor$ with $\qry_{i,t}$ and receive response $r_{i,j, t}$.\;\vspace{1mm}
   }
   Update $\hugap^{t}(\x_i) = \hugap^{t-1}(\x_i) - \ind[\frac{1}{J}\sum_j r_{i,j, t} < \frac{1}{2}]\cdot\frac{\utrue_{\sf max}}{2^t}.$
   }
	}
  \KwOut{Gap estimates $\hugap^T$}
	\caption{\algnoise: Robust \algdet for estimating $\utrgap$ with noisy oracle}
  \label{alg:stoch-bin}
\end{algorithm}
From an algorithmic perspective, it is well known that the perceptron algorithm itself is not noise-stable and can oscillate if there are datapoints $\x$ which have noisy labels. In order to overcome this limitation, several noise-robust perceptron variants have been proposed in the literature; see~\cite{khardon2007} for an extensive review.

We build on this line of work and present \algnoise (Algorithm~\ref{alg:stoch-bin}), a noise-robust variant of the deterministic \algdet algorithm.  The main difference is the presence of an additional inner-loop with index $j$ which repeatedly queries $\qry_{i,t}$ for $J = \tilde{O}\left(\frac{1}{(1-2\nq)^2} \right)$ times. In each iteration, the update is again a coordinate-wise perceptron update which matches the prediction of the current estimate with the average of the oracle responses. Such an averaging has been previously used in the context of learning halfspaces from noisy data both in a passive~\cite{bylander1994} and active~\cite{yan2017} framework.

The following lemma, whose proof we defer to Appendix~\ref{app:proofs-bin-finite}, provides an upper bound on the estimation error of the gap estimates produced by \algnoise.

\begin{lemma}[Estimation error of Algorithm~\ref{alg:stoch-bin}]\label{lem:perf-alg-stoch} Given access to datapoints $\xset = \{\x_1, \ldots, \x_\nsamp \}$ and noisy $\kor$-comparison oracle $\oracle_\kor$ satisfying Assumption~\ref{ass:noise-bnd} with parameter $\nq$, \algnoise (Algorithm~\ref{alg:stoch-bin}) uses $O\left(\frac{\nsamp}{(1-2\nq)^2}\cdot  \log \kor \cdot \log \frac{n\log \kor}{\conf}\right)$ queries and produces estimates $\hugap$ such that
  \begin{align}
    \max_{i \in [\nsamp]}\left\lvert\hugap(\x_i) - \utrgap(\x_i)\right\rvert   \leq \frac{2\utrue_{\sf{max}}}{\kor}\;,
  \end{align}
  with probability at least $1-\conf$.
\end{lemma}
In comparison to \algdet which requires $O(\nsamp \log \kor)$ queries to the comparison oracle, the robust variant \algnoise requires a fraction $\frac{1}{(1-2\eta)^2}$ more queries to achieve a similar estimation error. Such an increase in query complexity is typical of learning with such noisy oracles in the binary classification setup~\cite{balcan2007, balcan2013, dasgupta2009, yan2017}.

Similar to Corollary~\ref{cor:bin-det} in the previous section, we can combine the above high-probability bound on the estimation error to obtain a bound on the excess risk which scales as $\frac{1}{\kor}$ with the order $\kor$ of the comparison oracle.

\begin{corollary}\label{cor:bin-stoch}
  Consider the binary decision making problem with label space $\Y = \{0, 1\}$. Given $\nsamp$ datapoints $\{\x_1, \ldots, \x_\nsamp \}$ such that each $\x_i \sim \xdist$, the plug-in estimate $\hfkn$ from equation~\eqref{eq:plug-in}, when instantiated with the output of \algdet (Algorithm~\ref{alg:det-bin}), satisfies
    \begin{align*}
      \gap(\hfkn, \F; \utrue) \leq 2\cdot\sup_{\f \in \F}\left(|\util(\f;\utrue) - \hutil(\f;\utrue)|\right)  + \frac{2\utrue_{\sf max}}{\kor}\cdot \left(\frac{1}{\nsamp}\sum_{i=1}^\nsamp\ind[\ferm(\x_i) \neq \y_i]\right) \;.
    \end{align*}
    with probability at least $1-\conf$.
\end{corollary}
We omit the proof of this corollary since it essentially follows the same steps as that for Corollary~\ref{cor:bin-det}. This corollary establishes that by increasing the query complexity by a factor of $O\left(\nicefrac{1}{(1-2\nq)^2}\right)$, one can recover the same additive $\frac{1}{\kor}$ excess risk bound of the deterministic setup. Combined, Corollaries~\ref{cor:bin-det} and~\ref{cor:bin-stoch} establish the trade-offs in the reduction of the excess risk while eliciting more complex information about the underlying utility $\utrue$ through the $\kor$-comparison oracle.

\subsection{Information-theoretic lower bounds}\label{sec:bin-lower}
In the previous section, we studied the learning with unknown utility problem from an algorithmic perspective and showed that the plug-in estimator with \algdet estimates $\hu$ achieve an excess risk bound which scales as $O(\frac{1}{\kor})$ with the order $\kor$ of the comparison. In this section, we ask whether such a scaling of the error term is optimal and study this lower bound question from an information-theoretic perspective.

Recall from Theorem~\ref{thm:upper_finite} that the excess risk decomposes into two terms: (i) a uniform convergence term for the decision class $\F$ with respect to utility function $\utrue$ and (ii) an estimation error term corresponding to how well $\hu_\kor$ approximates $\utrue$ on the sampled datapoints. When the underlying utility function $\utrue$ is known, classical results from the learning theory literature the uniform convergence complexity term is in general unavoidable~\cite[see][Theorem 6.8]{shalev2014}. With this, we take the infinite-data limit, where the learner is assumed to have access to the distribution $\xdist$, and study whether the excess error of $O(\frac{1}{\kor})$ is necessary.

Our notion of minimax risk is based on the subset of utility functions which cannot be distinguished by any learner with access to a $\kor$-comparison oracle. Formally, given any oracle $\oracle_\kor(\cdot\;;\utrue)$, where we have made the dependence on the utility $\utrue$ explicit, we denote by $\uclass_{\kor, \utrue}$ the subset of utility functions in the class $\uclass$ which are consistent with the responses of $\oracle_\kor(\cdot\;;\utrue)$. With this, we define the information-theoretic minimax risk $\minmax_\kor(\F, \xdist)$ with respect to the function class $\F$ and distribution $\xdist$ as
\begin{align}\label{eq:mm-global}
  \minmax_\kor(\F, \xdist) \defn \sup_{\oracle_\kor(\cdot\;;\utrue)}\pinf_{\ld \in \simplex_\F} \sup_{\u \in\,\uclass_{\kor, \utrue}} \En_{\f \sim \ld}\left[\gap(\f, \F; \u)\right]\;,
\end{align}
where the infimum is taken over all procedures which take as input the distribution $\xdist$ over the instances and access to a $\kor$-comparison oracle, and output a possibly randomized estimate $\ld \in \simplex_\F$. The above notion of minimax risk can be viewed as a three-stage game between the learner and the environment. The sequence of supremum and infimum depicts the order in which information is revealed in this game. The environment first selects a $\kor$-query oracle $\oracle(\cdot\;;\utrue)$ with underlying utility $\utrue$. The learner is then provided access to the underlying distribution $\xdist$, function class $\F$ and the oracle $\oracle(\cdot\;;\utrue)$ based on which it outputs a possibly randomized decision function given by $\ld \in \simplex_\F$. The environment is then allowed to select the worst-case utility $\u$ such that it is consistent with the $\kor$-oracle $\oracle(\cdot\;;\utrue)$ and the learner is evaluated in expectation over this chosen utility.  We call this the minimax risk of learning with respect to class $\F$ and distribution $\xdist$.

Our next main result shows that there exist instances of the binary prediction problem $(\F, \xdist)$ such that the minimax risk $\minmax_\kor(\F, \xdist)$ is lower bounded by $\frac{1}{\kor}$ for any $\kor \geq 2$ up to some universal constants. Observe that this matches the corresponding upper bounds obtained in Corollaries~\ref{cor:bin-det} and~\ref{cor:bin-stoch} exhibiting that the proposed plug-in estimator in equation~\eqref{eq:plug-in} with \algdet (\algnoise for noisy oracle) utilities is indeed minimax optimal for the binary prediction setup.

\begin{theorem}\label{thm:bin-pred-lower}
There exists a universal constant $c > 0$ such that for any $\kor \geq 2$, there exist a binary prediction problem instance $(\F, \xdist)$ such that
\begin{align*}
  \minmax_\kor(\F, \xdist) \geq \frac{c}{\kor}\;.
\end{align*}
\end{theorem}
A few comments on Theorem~\ref{thm:bin-pred-lower} are in order. First, the above result shows a family of lower bounds for our learning with unknown utilities framework -- one for each value of the order $\kor$. Specifically, it shows that for every $\kor \geq 2$, there exists a worst-case instance such that any algorithm will incur an error of $\Omega(\frac{1}{\kor})$. Compare this with the upper bounds on excess risk from the previous section. In the limit of infinite data, Corollaries~\ref{cor:bin-det} and~\ref{cor:bin-stoch} exhibit that the excess risk \mbox{$\gap(\hfkn, \F;\utrue) = O(\frac{1}{\kor})$} for the plug-in estimator $\hfkn$. This establishes that the plug-in estimator with \algdet and \algnoise utility estimates is indeed minimax optimal.

\begin{proof}
In order to establish a lower bound on the minimax risk $\minmax_\kor$, we will construct two utility functions $\u_1, \u_2 \in \uclass$ such that the $\kor$-comparison oracle has identical responses for both these utility functions. For the purpose of our construction, we will consider noiseless oracle; the problem only becomes harder for the learner if the oracle responses are noisy. Given these two utility functions, we next show that their maximizers $\f_1$ and $\f_{2}$ are different for some function class $\F$. We then combine these two insights to obtain the final minimax bound.

For our lower bound construction, we will focus on a setup where the features are one dimensional with $\X = \real$ and the linear decision function class
\begin{align*}
   \Flin = \{\f_a\; |\; \f_a(\x) = \sign(a\x), \; a \in [-1,1] \}\;.
\end{align*}
Recall that for any point $\x$, we represent by $\ugap(\x) = \u(\x, \y_\x) - \u(\x, \bar{\y}_\x)$ the utility gain corresponding to the function $\u$. Before constructing the explicit example, we present a technical lemma which highlights a limitation of a $\kor$-comparison oracle -- it establishes that a $\kor$-oracle will not be able to distinguish utility functions for which the utility gaps are in the range $(1-\frac{1}{\kor}, 1)$.

\begin{lemma}\label{lem:k-lim}
Consider any utility functions $\u_1, \u_2 \in \uclass$. Let datapoints $\x$ have utility gain $\ugap^i(\x)$ for $i = \{1,2\}$. For any two points $\x_1, \x_2$ such that
\begin{align*}
  \ugap^1(\x_1) = \ugap^2(\x_1) = \ugap(\x_1) \quad \text{and} \quad \left(1-\frac{1}{\kor}\right)\cdot \ugap(\x_1)\leq \ugap^i(\x_2) \leq \ugap(\x_1)\;,
\end{align*}
the oracle responses for any query $\query = (\bx, \by_1, \by_2)$ comprising points $\x_1$ and $\x_2$ are identical for $\utrue = \u_1$ or $\utrue = \u_2$.
\end{lemma}
We defer the proof of the above lemma to Appendix~\ref{app:proofs-bin-finite}. Taking this as given, we proceed with our lower bound construction.\\

\noindent\textbf{Utility functions $\u_1$ and $\u_2$.} Our construction considers two datapoints $\x_+ = +1$ and $\x_- = -1$ and two utility functions $\u$ and $\tilde{\u}$ satisfying
\begin{align*}
  \u_1(\x_+, 1) > \u_1(\x_+, 0) \quad &\text{and} \quad \u_1(\x_-, 1) > \u_1(\x_-, 0)\;,\\
  \u_2(\x_+, 1) > \u_2(\x_+, 0) \quad &\text{and} \quad \u_2(\x_-, 1) > \u_2(\x_-, 0).
\end{align*}
Observe that under these utilities, any function $\f_a \in \Flin$ can make a correct decision for either point $\x_+$ or point $\x_-$ but not for both simultaneously. Given these datapoints, the two utility functions are given by
\begin{align*}
  \u_1( \x_+, 1) = 1, \quad \u_1(\x_-, 1) = 1 - \gamma_1 \; \text{ where } \gamma_1=\frac{1}{2(3\kor+1)}\\
  \u_2( \x_+, 1) = 1, \quad \u_2(\x_-, 1) = 1 - \gamma_2 \; \text{ where } \gamma_2=\frac{2}{(3\kor+1)},
\end{align*}
and $\u_i(\x, 0) = 0$ for both $i = \{1,2 \}$. Observe that both $\gamma_1, \gamma_2$ have been set to satisfy the conditions of Lemma~\ref{lem:k-lim}, that is,
\begin{align*}
  \left(1 - \frac{1}{k}\right)\cdot \ugap(\x_+) \leq \ugap^i({\x_-}) \leq \ugap(\x_+) \text{ for } i = \{1,2\}.
  \end{align*}\\

\noindent\textbf{Distribution $\xdist$.} For any $\kor > 2$, consider the distribution $\xdist$ over the points $\{\x_+, \x_-\}$ such that
\begin{align*}
  \Pr(\x = \x_+) = \frac{3\kor}{6\kor+1}\quad \text{ and }\quad  \Pr(\x = \x_-) = \frac{3\kor+1}{6\kor+1}.
  \end{align*}
By Lemma~\ref{lem:k-lim}, we have that using the $\kor$-comparison oracle, no learner can distinguish between the utility functions $\u_1$ and $\u_2$ on the distribution $\xdist$. Further, recall that any classifier $\f_a \in \Flin$ can either predict $\x_+$ or $\x_-$ correctly. We now obtain a bound on the excess risk $\gap(\f_a, \F;\u)$ for both these cases separately.\\

\noindent \underline{Case 1}: $\f_a(\x_+) = 1$. In this case, the utility gap is maximized by setting the utility $\u = \u_1$  in the minimax risk. The corresponding excess risk is given by
\begin{align}\label{eq:gap-lin-c1}
  \gap(\f_a, \F;\u_1) = \frac{(3\kor+1)(1-\gamma_1)}{6\kor + 1} - \frac{3\kor}{6\kor+1} = \frac{1}{2(6\kor+1)}.
\end{align}

\noindent\underline{Case 2}: $\f_a(\x_-) = 1$. In this case, the utility gap is maximized by setting the utility $\u = \u_2$  and the excess risk is given by
\begin{align}\label{eq:gap-lin-c2}
  \gap(\f_a, \F;\u_2) = \frac{3\kor}{6\kor+1} - \frac{(3\kor+1)(1-\gamma_2)}{6\kor + 1}= \frac{1}{(6\kor+1)}.
\end{align}
Noting that any predictor $\hat{\f}$ will output a function corresponding to one of the two cases above and combining equations~\eqref{eq:gap-lin-c1} and~\eqref{eq:gap-lin-c2} establishes the desired claim.
\end{proof}
While the information theoretic results of this section showed that the plug-in estimator is minimax optimal, the next section focuses on whether this estimator is able to \emph{adapt} to easier problem instances -- specifically, whether our estimation procedures \algdet and \algnoise are optimal for every problem instance? We answer this in the negative and introduce a new estimator which is instance optimal. However, such an adaptivity to easier instances comes at the cost of an exponential query complexity.

%% file: local-minimax.tex
\section{Instance-optimal guarantees for binary prediction}\label{sec:local-minimax}

In the previous section, we proposed query-efficient algorithms, \algdet and \algnoise, for learning a function $\hfkn$ with small excess risk using only $\tilde{O}(\nsamp \log \kor)$  queries to the $\kor$-comparison oracle. Further, the upper bounds in Corollaries~\ref{cor:bin-det} and~\ref{cor:bin-stoch} along with the lower bound of Theorem~\ref{thm:bin-pred-lower} establish that our proposed algorithms are indeed minimax optimal over the class of utility functions~$\uclass$. Given this, it is natural to ask whether our proposed algorithms are instance wise-optimal, that is, do they achieve the best possible excess-risk bounds for \emph{all} $\utrue \in \uclass$?

To simplify our presentation, we study this question at the population level,\footnote{Our analysis could be extended to the finite sample setup using the bound obtained in Theorem~\ref{thm:upper_finite}.} where we assume that the learner has access to the underlying distribution $\xdist$. This allows us to focus on the excess risk as a function of the order $\kor$ of the comparison oracle and ignore the uniform convergence term. We also restrict our attention to the deterministic noiseless oracle since one can reduce the noisy oracle to the noiseless oracle by using the averaging technique presented in Section~\ref{sec:bin-upper}.

The following proposition shows that the plug-in estimator with \algdet utilities are \emph{not instance-optimal}, that is, it does not adapt to the hardness of the learning with unknown utilities problem instance. Specifically, it constructs a problem instance $(\F, \xdist)$ with a noiseless oracle and shows that the estimate\footnote{Since we are working at the population level, we have dropped the subscript $\nsamp$ from $\hfkn$} $\hfk$ from equation~\eqref{eq:plug-in} with \algdet utility estimates has an excess risk of $\frac{1}{\kor}$ while there exists an estimator, which uses all $\kor$-queries and is able to achieve zero excess risk. 

Recall that for any utility $\utrue \in \uclass$, we denote by $\uclass_{\kor, \utrue}$ the subset of utility functions in the class $\uclass$ which are indistinguishable from $\utrue$ under the $\kor$-comparison oracle $\oracle(\cdot\;;\utrue)$.

\begin{proposition}[Plug-in with \algdet estimates is not instance-optimal]\label{prop:subopt-plugin}
For every $\kor > 2$, there exists an binary prediction instance $(\F, \xdist)$ along with an oracle $\oracle_\kor$ such that
\begin{itemize}
  \item[a)] The error of the plug-in estimate $\hfk$ from equation~\eqref{eq:plug-in} with estimated utilities $\hu_\kor$ from \algdet (Algorithm~\ref{alg:det-bin}) is non-zero, that is,
  \begin{equation*}
    \gap(\hfk, \F; \utrue) = \frac{1}{\kor}.
  \end{equation*}
  \item[b)] There exists an optimal predictor $\tf$ with zero excess-risk, that is,
  \begin{equation*}
    \sup_{\u \in \uclass_{\kor, \utrue}}\gap(\tf, \F;\u) = 0.
  \end{equation*}
\end{itemize}
\end{proposition}

We make a few remarks about the proposition. While the first part of the proposition shows that the excess risk $\gap(\hfk, \F;\utrue) = \frac{1}{\kor}$, the second part makes a stronger claim about the performance of $\tf$ on all utilities $\u \in \uclass_{\kor, \utrue}$. This shows that the predictor $\tf$ performs well when evaluated on an entire neighborhood around the true utility $\utrue$. We defer the proof of the proposition to Appendix~\ref{app:proofs-minmax}.

Having established that our estimators from the previous section are not adaptive, we introduce a notion of \emph{local minimax risk} and study estimators which are instance-optimal. We begin by precisely defining this notion of instance-wise minimax optimality. Recall from Section~\ref{sec:bin-lower}, our notion of minimax risk $\minmax_\kor(\F, \xdist)$ was a worst-case notion -- the minimax risk was defined as a supremum over all oracles $\oracle_\kor(\cdot\;;\utrue)$. We extend this global minimax notion to a local minimax one. In particular, for any $\utrue \in \uclass$, we define the local minimax risk around $\utrue$ as
\begin{align}
  \minmax_\kor(\F, \xdist;\utrue) \defn \inf_{\hf} \sup_{\u \in \,\uclass_{| \utrue}}\left[\gap(\hf,\F;\u) \right]\;,
\end{align}
where the infimum is again over the set of all estimators which output a function $\hf \in \F$ given access to distribution $\xdist$ and $\kor$-comparison oracle $\oracle_\kor$. Observe that this local notion of minimax risk concerns the performance of an algorithm $\hf$ around a specific instance $\utrue$ as compared to the worst-case instance.

For any utility function $\u \in \uclass$, we define its population maximizer $\f_\u \in \argmax_{\f\in \F} \util(\f;\u)$. With this notation, our next theorem provides a lower bound on this local minimax risk in terms of a local modulus of continuity with respect to the set $\uclass_{\kor, \utrue}$.

\begin{theorem}[Local minimax lower bound]\label{thm:local-lower}
  For any distribution $\xdist$  over feature space $\X$, utility function $\utrue \in \uclass$, function class $\F$ and order $\kor$ of the comparison oracle, the local minimax risk
    \begin{align}
      \minmax_\kor(\F, \xdist;\utrue) \geq  \frac{1}{2} \cdot \sup_{\u_1, \u_2 \in\; \uclass_{\kor, \utrue}} \left( \util(\f_{\u_1};\u_1) - \util(\f_{\frac{\u_1+\u_2}{2}};\u_1) \right).
    \end{align}
\end{theorem}
\begin{proof}
  Consider any two utility functions $\u_1, \u_2 \in \uclass_{\kor, \utrue}$ and let $\bar{u} = \frac{\u_1+\u_2}{2}$. We can then lower bound the minimax risk as
  \begin{align*}
    \minmax_\kor(\F, \xdist;\utrue) &\geq \inf_{\f \in \F} \left(\frac{1}{2}\gap(\f, \F;\u_1) + \frac{1}{2} \gap(\f, \F;\u_2) \right)\\
    &=\frac{1}{2}\gap(\f_{\bar{u}}, \F;\u_1) + \frac{1}{2} \gap(\f_{\bar{u}}, \F;\u_2)\\
    &\geq \frac{1}{2} \left(\util(\f_{\u_1};\u_1) - \util(\f_{\bar{\u}};\u_1) \right),
  \end{align*}
  where the last equality follows by noting that $\gap(\f_{\bar{u}}, \F;\u_2) \geq 0$. Since the above holds for any choice of $\u_1, \u_2$, the desired bound follows by taking a supremum over these values.
\end{proof}
A few comments on Theorem~\ref{thm:local-lower} are in order. The theorem establishes that the local minimax risk $\minmax_{\kor}(\F, \xdist)$ is lower bounded by a local modulus of continuity, 
\begin{align}
    \sup_{\u_1, \u_2 \in\; \uclass_{\kor, \utrue}} \left( \util(\f_{\u_1};\u_1) - \util(\f_{\frac{\u_1+\u_2}{2}};\u_1) \right)\;,
\end{align}
which captures the worst-case variation in the performance of utility maximizers of utility in a neighborhood of $\utrue$. For any two utilities $\u_1, \u_2 \in \uclass_{\kor, \utrue}$, it measures the performance drop in the utility of a learner uses the maximizer $\f_{\frac{\u_1+\u_2}{2}}$ in place of $\f_{\u_1}$ when the underlying utility is $\u_1$.

Given this lower bound on the local minimax risk $\minmax_{\kor}(\F, \xdist)$, it is natural to ask whether this local modulus of continuity exactly captures the instance-specific hardness of the problem. To this end, our next result answers this in the affirmative. In particular, it shows that for any $\utrue$, the randomized minimax robust estimator $\pfrob \in \simplex_\F$, given by
\begin{align}\label{eq:f-rob}
  \pfrob  \in  \argmin_{p \in \simplex_\F} \sup_{\u \in \uclass_{\kor, \utrue}}\En_{\f \sim p}[\gap(\f, \F;\u)],
\end{align}
(nearly-)obtains the same excess-risk bound as that given by the lower bound in Theorem~\ref{thm:local-lower}.
\begin{theorem}[Upper bounds for $\pfrob$]\label{thm:local-upper} For any distribution $\xdist$  over feature space $\X$, utility function $\utrue \in \uclass$ and function class $\F$, the expected excess risk of the randomized estimator given by the distribution $\pfrob \in \simplex_\F$ is
  \begin{align}\label{eq:local-upper}
    \En[\gap(\pfrob, \F;\utrue)] &= \sup_{\ld_\u}\left( \En_{\u' \sim \ld_u} \left[\util(\f_{\u'};\u') - \util(\f_{\ld_u};\u') \right]\right)\nonumber \\
    &\leq \sup_{\u_1, \u_2 \in\; \uclass_{\kor, \utrue}} \left( \util(\f_{\u_1};\u_1) - \util(\f_{\u_2};\u_1) \right),
  \end{align}
  where the distribution $\ld_\u \in \simplex_{\uclass_{\kor, \utrue}}$ is over the space of utility functions consistent with $\utrue$.
\end{theorem}
We defer the proof of Theorem~\ref{thm:local-upper} to Appendix~\ref{app:proofs-minmax}. Compared with the lower bound of Theorem~\ref{thm:local-lower}, the bound in~\eqref{eq:local-upper} shows that the local minimax risk can indeed be upper bounded by a similar local modulus of continuity. Observe that the while the lower bound evaluates the performance loss of the maximizer $\f_{\frac{\u_1 + \u_2}{2}}$, the upper bound is evaluated on $\f_{\u_2}$. While the minimax estimator $\pfrob$ in equation~\eqref{eq:f-rob} is defined at the population level, we can naturally extend it to the finite sample regime as
\begin{align}
    \hat{\ld}_{{\sf rob}, \nsamp} \in \argmin_{\ld \in \simplex_\F} \sup_{u \in \hat{\uclass}_{\kor, \utrue}} \En_{\f \sim \ld}[\hutil(\f_\u;\u) - \hutil(\f;\u)]
\end{align}
where the class of utilities $\hat{\uclass}_{\kor, \utrue}$ represents the set of all $n$-dimensional vectors in $[0,1]^\nsamp$ which are consistent with responses to all $\kor$-queries on the set of sampled datapoints $\xset$. Using a similar analysis as in Theorem~\ref{thm:upper_finite}, one can then upper bound the excess risk of this estimator in terms of the local modulus on the dataset $\xset$ and an additional uniform convergence term. 

In comparison to the \algdet procedure which uses $O(\nsamp\log \kor)$ queries to the comparison oracle for estimating utilities, the estimator $\hat{\ld}_{{\sf rob}, \nsamp}$ uses $O(\nsamp^\kor)$ queries to construct the set $\hat{\uclass}_{\kor, \utrue}$. Thus, while this estimator adapts to the problem hardness, such an adaptation comes at the cost of an exponential increase in query complexity. Achieving instance-optimality by using fewer queries is an interesting question for future research. 

%% file: acks.tex
\section*{Acknowledgments}
We thank members of the InterACT lab and Steinhardt group for helpful discussion and feedback.

KB is supported by a JP Morgan AI Fellowship. This work was partially supported  by Office of Naval Research Young Investigator Award, NSF CAREER and a AFOSR grant to ADD, and by a grant from Open Philanthropy to JS.

%% file: rel.tex
\section{Other related work}
This paper sits at the intersection of multiple fields of study: agnostic learning , learning with nuisance parameters, and utility learning from preferences .
Here, we review the papers that are most relevant to our contributions. 

\paragraph{Agnostic learning.} The framework of probably approximately correct (PAC) learning was introduced in their seminal work by Valiant~\cite{valiant1984}. This framework formalized the problem of learning from sampled data in a realizable setup. This was formally extended to the agnostic setup, with no assumptions on the data generating distribution, by Haussler~\cite{haussler1992}. Connections of learnability with uniform convergence were first established by Vapnik~\cite{vapnik1992}, and more recently it was established in~\cite{shalev2014} that for the general learning problem, such a uniform convergence is not necessary to establish learnability. Similar to the classical agnostic supervised learning, the learner does not know the distribution $\xdist$ but only has access to it via samples. The key difference is that the classical setup assumes that the utility function $\utrue$ is known to the learner while our framework does not.

\paragraph{Learning with nuisance parameters.} Closely related to our setup is the problem of learning with a nuisance component~\cite{foster2019} which comprises as special case the problems of heterogeneous treatment effect estimation~\cite{chernozhukov2017}, offline policy learning~\cite{athey2017}, and learning with missing data~\cite{graham2011} amongst others. In this setup, objective is to learn a predictor with small excess risk and this risk depends on a underlying nuisance parameter which is unknown to the learner a priori. The unknown utility $\utrue$ of our setup can be seen as a nuisance component in their framework. However, the two problems differ in the form of information available to the learner -- they allow the learner to directly elicit (possibly noisy) values of utility $\utrue$. They additionally require that utility $\utrue$ belongs to some pre-specified function class and their bounds depend on the rate at which this utility function is learnable over this class.

Another line of work, called double/debiased machine learning in the statistics and econometrics literature~\cite{chernozhukov2018a,chernozhukov2018b,chernozhukov2018c}, addresses semiparametric inference~\cite{robinson1988, kosorok2007} where the function class $\F$ is assumed to be a parametric family along with a non-parametric nuisance component. In addition to the differences mentioned above, this class of methods focuses on exact parameter recovery and conditions under which $\sqrt{n}$-consistent and asymptotically normal estimators can be obtained.

\paragraph{Utility estimation with preferences.} The seminal work of von Neumann and Morgenstern~\cite{morgenstern1953} established that any rational agent whose preferences satisfy certain axioms will have a utility function. Furthermore, the proof of this expected utility theorem showed these utilities could be elicited from the agent using preferences over randomized lotteries. As discussed in Section~\ref{sec:intro},  such preferences over lotteries can be seen as a special case of the $\infty$-comparison oracle. There have been several recent works studying the consequences of incomplete preferences~\cite{ok2002, galaabaatar2013} which show the existence of a class of utility functions which are consistent with these incomplete preferences. Our $\kor$-comparison oracles can be seen as a quantitative approach to studying such incomplete preferences; for each value of $\kor \geq 1$, the human expert can only compare lotteries up to a granularity of $\frac{1}{\kor}$. Our work goes a step forwards and studies the consequences of such incomplete preferences for decision-making tasks.

%% file: app_defproofA.tex
\section{Deferred proofs from Section~\ref{sec:bin-pred}}\label{app:proofs-bin-finite}
\subsection{Proof of Proposition~\ref{prop:upper_bin}}
The first part of the proof essentially follows the same as that for Theorem~\ref{thm:upper_finite}. The proof differs in how we upper bound Term (I) from equation~\eqref{eq:error-decomp-upper}.
\begin{align}
  \hutil(\ferm;\utrue) - \hutil(\hfkn;\utrue) &\leq \frac{1}{\nsamp} \sum_{i=1}^\nsamp (\ind[\ferm(\x_i) = \y_i] - \ind[\hfkn(\x_i) = \y_i])(\utrgap(\x_i) - \hugap(\x_i))\nonumber\\
  &= \frac{1}{\nsamp} \sum_{i=1}^\nsamp (\ind[\hfkn(\x_i) \neq \y_i] - \ind[\ferm(\x_i) \neq \y_1])(\hugap(\x_i) - \utrgap(\x_i))\\
  &\stackrel{\1}{\leq}\frac{1}{n} \sum_{i=1}^\nsamp \ind[\ferm(\x_i) \neq \y_i](\utrgap(\x_i) - \hugap(\x_i))\nonumber\\
  &\leq \max_{i}[\utrgap(\x_i) - \hugap(\x_i)]\cdot \frac{1}{\nsamp}\sum_{i=1}^\nsamp\ind[\ferm(\x_i) \neq \y_i],
\end{align}
where inequality $\1$ follows from the fact that $\hu$ is a lower estimate of $\utrue$. This establishes the desired claim.
\qed
\subsection{Proof of Lemma~\ref{lem:perf-alg-det}}
We begin by noting that for any given datapoint $\x_i$, the deterministic comparison oracle $\oracle_\kor$ when queried with $\query_{i,t}$  outputs
  \begin{align*}
    \oracle_\kor(\query_{i,j}) = \ind\left[\frac{k}{2}\utrgap(\x_i) \geq \lambda\utrue_{\sf max} \right]\;,
  \end{align*}
for values\footnote{We denote by $[d]$ the set of integers $\{1, \ldots, d\}$.} of $\lambda \in [\frac{k}{2}]$.  This effectively allows one to compare the utility gap $\utrue_i$ with $\utrue_{\sf max}$ at a multiplicative granularity of $\frac{2}{k}$. With this observation, let us establish that for any time $t \in [T]$, for any datapoint $\x_i \in \xset$, we have
\begin{align}\label{eq:to-prove}
\hugap^{t}(x_i) - \frac{\utrue_{\sf max}}{2^t} \leq \utrgap(x_i) \leq \hugap^{t}(x_i).
\end{align}
The proof will proceed via an inductive argument.\\

\noindent \underline{Base Case}. For initial time $t = 0$, by the boundedness of the utility functions, we have for all $\x_i$,
\begin{align*}
\hugap^{0}(x_i) - \utrue_{\sf max} = 0 \leq \utrgap(x_i) \leq \utrue_{\sf max} = \hugap^{0}(x_i).
\end{align*}

\noindent \underline{Induction Step}. Assume that for some $t = s$, equation~\eqref{eq:to-prove} holds for all $\x_i \in \xset$. We will now show that it holds for $t = s+1$. Note that by the induction hypothesis, the value of $\lambda$ at time $s+1$ can be equivalently written as
\begin{align*}
  \lambda = \frac{k}{2\utrue_{\sf max}}\cdot\left(\frac{\hugap^{s}(x_i) - \frac{\utrue_{\sf max}}{2^s} + \hugap^{s}(x_i)}{2} \right)\;,
  \end{align*}
  that is, as a scaled mid-point of the confidence interval at time $s$. the query $\qry_{i,t}$ then compares the gap $\utrgap(\x_i)$ with the mid-point of the confidence interval.\\

   \noindent \emph{Case 1.} If the response $r_{i,t}=1$ which implies that $\hugap^s(\x_i) \geq \frac{2\lambda}{\kor}$, the upper estimate remains the same and the lower estimate is (implicitly) moved to the mid-point $\frac{2\lambda}{\kor}$ since we know from the oracle's response that $\utrgap(\x_i)$ is greater than the mid-point. Thus, after each update, the confidence interval shrinks by a factor of $\frac{1}{2}$ and reduces to $\frac{1}{2^{s+1}}$ at the end of time $t = s+1$.\\

   \noindent \emph{Case 2.} On the other hand if $r_{i,t} = 0$, the estimate $\hugap^{s+1}(\x_i)$ is updated to be the midpoint $\frac{\utrue_{\sf max}}{2^{s+1}}$ while the lower estimate remains the same because of the oracle's response.\\

Combining both the cases above, we see that at time $t=s+1$, the confidence interval for $\utrgap(\x_i)$ is exactly $\frac{1}{2^{s+1}}$ for both the cases. Thus, we must have that
\begin{align*}
\hugap^{s+1}(x_i) - \frac{\utrue_{\sf max}}{2^{s+1}} \leq \utrgap(x_i) \leq  \hugap^{s+1}(x_i).
\end{align*}
This establishes the first part of the claim. The bound on the query complexity follows from the fact that for each datapoint $\x_i$, we use $\log_2 \kor - 1$ queries to the oracle in the procedure. This establishes the desired claim. \qed

\subsection{Proof of Corollary~\ref{cor:bin-det}}
  The excess risk of the plug-in estimator can be upper-bounded from Proposition~\ref{prop:upper_bin} as
  {\small
  \begin{align}
    \gap(\hfkn, \F; \utrue) &\leq 2\cdot\sup_{\f \in \F}\left(|\util(\f;\utrue) - \hutil(\f;\utrue)|\right)  + \max_{i}[\utrgap(\x_i) - \hugap(\x_i)]\cdot \frac{1}{\nsamp}\sum_{i=1}^\nsamp\ind[\ferm(\x_i) \neq \y_i]\nonumber \\
    &\stackrel{\text{Lemma}~\ref{lem:perf-alg-det}}{\leq} 2\cdot\sup_{\f \in \F}\left(|\util(\f;\utrue) - \hutil(\f;\utrue)|\right) +  \frac{2\utrue_{\sf max}}{\kor} \cdot \frac{1}{\nsamp}\sum_{i=1}^\nsamp\ind[\ferm(\x_i) \neq \y_i]\;,
  \end{align}}
  where the last inequality follows by noting that \algdet produces an estimate $\hugap$ of the utility gap $\utrgap(\x_i)$ with an additive error of $\frac{2\utrue_{\sf max}}{\kor}$.
\qed

\subsection{Proof of Lemma~\ref{lem:perf-alg-stoch}}
  To establish the above claim, we show that the updates to the gap estimates $\hugap^{t}$ performed by \algnoise mirror those performed by the deterministic \algdet with high probability.
    For any datapoint $\x_i$ and any time $t \in [T]$, denote by $r^*_{i,t} = \En\left[r_{i,j,t} \right]$ the expected value of the response for query $\qry_{i, j, t}$. By Assumption~\ref{ass:noise-bnd}, we have that $\ind[r^*_{i,t} < \frac{1}{2}]$ provides the true label for the query $\qry_{i,j,t}$. By an application of the Hoeffding's inequality, we have,
  \begin{align*}
    \Pr\left(\ind\left[\frac{1}{J}\sum_{j}r_{i, j, t} < \frac{1}{2}\right] \neq \ind \left[r^*_{i,t}< \frac{1}{2}\right] \right) \leq \exp\left(\frac{-J(1-2\nq)^2}{4} \right).
  \end{align*}
  Taking a union bound over all datapoints $\x_i$ and time $t \in [T]$, and substituting the value of $J = \frac{8}{(1-2\nq)^2}\log(\frac{\nsamp T}{\conf})$, we have,
  \begin{align}
    \Pr\left(\exists i,t \; \text{s.t.} \;\ind\left[\frac{1}{J}\sum_{j}r_{i, j, t} < \frac{1}{2}\right] \neq \ind \left[r^*_{i,t}< \frac{1}{2}\right] \right) \leq \conf.
  \end{align}
From the above equation, we have that with probability at least $1-\conf$, every update performed by \algnoise uses the correct label. Combining the above with the proof of Lemma~\ref{lem:perf-alg-det} establishes the required claim.
\qed

\subsection{Proof of Lemma~\ref{lem:k-lim}}
Observe that from the conditions of the lemma statement, we have
\begin{equation*}
\left(1 - \frac{1}{\kor} \right)\cdot \ugap(\x_1) \leq \ugap^1(\x_2), \ugap^2(\x_2) \leq \ugap(\x_1).
\end{equation*}
Assume without loss of generality that $\ugap^1(\x_2)> \ugap^2(\x_2)$. Observe that any $\kor$-query comprising only points $\x_1$ and $\x_2$ must have the form
\begin{equation*}
  \bx = (\underbrace{\x_1, \ldots, \x_1}_{j_1 \text{ times}}, \underbrace{\x_2, \ldots, \x_2}_{j_2 \text{ times}}),\quad \by_1= (\underbrace{\y_1, \ldots, \y_1}_{j_1 \text{ times}}, \underbrace{\bar{\y}_2, \ldots, \bar{\y}_2}_{j_2 \text{ times}}), \quad \by_1= (\underbrace{\bar{\y}_1, \ldots, \bar{\y}_1}_{j_1 \text{ times}}, \underbrace{{\y}_2, \ldots, {\y}_2}_{j_2 \text{ times}})
\end{equation*}
with $j_1 + j_2 = k$. For any query $\qry$ to be different under the oracles $\oracle_\kor(\cdot\;;\u_1)$ and $\oracle_\kor(\cdot\;;\u_2)$, we should have
\begin{align*}
\ind[j_1\ugap^1(\x_1) > j_2\ugap^1(\x_2)] = 1 \quad \text{and} \quad \ind[j_1\ugap^2(\x_1) > j_2\ugap^2(\x_2)] = 0\;
\end{align*}
since $\ugap^1(\x_2)> \ugap^2(\x_2)$. In order for the above equation to be satisfied, we requires that the ratio $\frac{j_1}{j_2} \geq 1 - \frac{1}{\kor}$. However, under the constraints $j_1 + j_2 = \kor$, this is not possible. Hence, it is not possible to distinguish between the utilities $\u_1$ and $\u_2$ using a $\kor$ comparison oracle. 
\qed

%% file: app_defproofB.tex
\section{Deferred proofs from Section~\ref{sec:local-minimax}}\label{app:proofs-minmax}
\subsection{Proof of Proposition~\ref{prop:subopt-plugin}}
Our example construction will focus on the real-valued feature space $\X = \real$, binary decision space $\Y = \{0,1 \}$, and the class of linear decision functions
  \begin{align*}
    \Flin = \{\f_a\; |\; \f_a(\x) = \sign(a\x), \; a \in [-1,1] \}\;.
  \end{align*}
  \paragraph{Distribution $\xdist$.} Our example will focus on three points $\x_1 = 1, \x_2 = 2, \x_3  = -1$ with their population probabilities given by
  \begin{align*}
    \Pr(\x = \x_1) = p, \quad Pr(\x = \x_2) = p, \quad \text{and} \quad \Pr(x = \x_3) = 1-2p\;,
  \end{align*}
  for some value $p > 0$ which we define later. Note that our final choice of $p$ will depend on the order $\kor$ of the comparison oracle.
 \paragraph{Utility function $\utrue$.} Given the above three points, we set the utility $\utrue(\x_i, 0) = 0$ for all datapoints $\x_i$. The utilities for label $\y = 1$ are given by
 \begin{align*}
   \utrue(\x_1, 1) = 1, \quad \utrue(\x_2, 1) = \frac{4}{\kor}, \quad \text{and} \quad \utrue(\x_3, 1) = \frac{2}{\kor^2}.
 \end{align*}
 With these utilities, observe that the true label $\y_i = 1$ for all the datapoints. Further, any predictor $\f \in \Flin$ can either correctly predict the points $\{\x_1, \x_2\}$ or the point $\x_3$ but not all three simultaneously.

 \paragraph{Performance of predictors.} For this setup described above, we now proceed to describe the optimal function $\f^*$, the plug-in estimate $\hfk$ and an alternate predictor $\tf$ which outperforms the plug-in estimate. Observe that any estimator will pick either $\f_{-1}$ or $\f_{+1}$ depending on the value of $p$.\\

 \noindent \underline{Optimal Classifier.} The difference in the expected utility between the classifiers $\f_{+1}$ and $\f_{-1}$ is given by
 \begin{align}
   \util(\f_{+1};\utrue) - \util(\f_{-1};\utrue) &= p\cdot \left(1+\frac{4}{\kor} \right) - (1-2p)\cdot\left(\frac{2}{\kor^2} \right)\nonumber \\
   &= \frac{p}{\kor^2}\cdot (\kor+2)^2 - \frac{2}{\kor^2}\nonumber\\
   &= \frac{1}{\kor^2}\left(p(\kor+2)^2 - 2\right)\nonumber.
 \end{align}
 Given the calculation above, the optimal classifier $\f^*$ is given by
 \begin{align}
   \f^* = \begin{cases}
   \f_{+1} \quad &\text{ for } p\geq \frac{2}{(\kor+2)^2}\\
   \f_{-1} \quad &\text{ otherwise}
 \end{cases}.
 \end{align}

\noindent \underline{Plug-in estimate $\hfk$.} We now study the prediction $\hfk$ obtained by using the prediction $\hu$ from \algdet (Algorithm~\ref{alg:det-bin}). Recall that since \algdet produces upper estimates for $\utrgap$ (which is equivalent to $\utrue$ since $\utrue(\x, 0) = 0$) within an error of $\frac{2}{\kor}$, the output estimates will be
\begin{align*}
  \hu(\x_1, 1) = 1, \quad \hu(\x_2, 1) = \frac{4}{\kor}, \quad \text{and}\quad \hu(\x_3, 1) = \frac{2}{\kor}.
\end{align*}
Observe that while \algdet is able to correctly learn the utilities for $\x_1$ and $\x_2$, it overestiamtes the utility for the point $\x_3$. Let us look at the difference of estimated utilities
\begin{align}
  \util(\f_{+1};\hu) - \util(\f_{-1};\hu) &= p\cdot \left(1+\frac{4}{\kor} \right) - (1-2p)\cdot\left(\frac{2}{\kor} \right)\nonumber\\
  &= \frac{p}{\kor}\cdot (k+8) -\frac{2}{\kor}\nonumber\\
  &= \frac{1}{\kor}\cdot(p(\kor+8) - 2)\nonumber.
\end{align}
Given the above calculations, we see that the function $\hfk$ is given by
\begin{align}
  \hfk = \begin{cases}
  \f_{+1} \quad &\text{ for } p\geq \frac{2}{\kor+8}\\
  \f_{-1} \quad &\text{ otherwise}
\end{cases}.
\end{align}

\noindent \underline{Alternate estimator $\tf$.} While \algdet compares the utilities of both  $\x_2$ and $\x_3$ with respect to $\x_1$ (equivalently $\x_{\imax}$), consider the alternate procedure which differs in the estimation of  utility gap $\utrgap(\x_3)$. Instead of using the proposed queries $\query_{3,t}$ of \algdet, we modify those as $\tilde{\qry}_{3,t} = (\bx, \by_1, \by_2)$ where
\begin{align*}
  \bx = (\underbrace{\x_{3}, \ldots, \x_{3}}_{\frac{k}{2} \text{ times }}, \underbrace{\x_{2}, \ldots, \x_{2}}_{\lambda \text{ times }}), \quad \by_1 = (\underbrace{\y_{3}, \ldots, \y_{3}}_{\frac{k}{2} \text{ times }}, \underbrace{1-\y_{2}, \ldots, 1- \y_{2}}_{\lambda \text{ times }}),\quad \by_2 = 1 - \by_1.
\end{align*}
Following the same proof as of Lemma~\ref{lem:perf-alg-det}, we can show that one can obtain an upper estimate $\tu(\x_3, 1) = \frac{8}{\kor^2}$. This follows from the fact that we can deduce that $\utrue(\x_3) \in [0, \frac{2\utrue(\x_2)}{\kor}]$ from the above queries and combining this with the fact that $\utrue(\x_2) \leq \frac{4}{\kor}$. Evaluating the difference between the utilities with respect to $\tu$, we get
\begin{align}
  \util(\f_{+1};\tu) - \util(\f_{-1};\tu) &= p\cdot \left(1+\frac{4}{\kor} \right) - (1-2p)\cdot\left(\frac{8}{\kor^2} \right)\nonumber \\
  &= \frac{p}{\kor^2}\cdot ((\kor+2)^2+12) - \frac{8}{\kor^2}\nonumber\\
  &= \frac{1}{\kor^2}\left(p((\kor+2)^2+12) - 8\right)\nonumber.
  \end{align}
Using such estimates $\tu$ with the plug-in estimator in equation~\eqref{eq:plug-in}, we have that the function
\begin{align}
  \tf = \begin{cases}
  \f_{+1} \quad &\text{ for } p\geq \frac{8}{(\kor+2)^2+12}\\
  \f_{-1} \quad &\text{ otherwise}
\end{cases}.
\end{align}

Thus, the three estimators $\f^*$, $\hfk$ and $\tf$ differ in the threshold for $p$ for switching between the functions $\f_{+1}$ and $\f_{-1}$. Setting a value of $p = \frac{1}{\kor+8}$, we see that for $\kor > 10$
\begin{equation*}
  \frac{2}{(\kor+2)^2} < \frac{8}{(\kor+2)^2 + 12} < \underbrace{\frac{1}{\kor+8}}_{p} < \frac{2}{\kor+8}.
\end{equation*}
Thus, for this setting of $p$, while the predictor $\f^* = \tf = \f_{+1}$, the estimator $\hfk = \f_{-1}$ and hence it incurs an excess risk $\gap(\hfk, \F;\utrue) = \frac{1}{k}$. This establishes the first part of the claim.

For the second part, observe that the estimator $\tilde{\f}$ outputs $\f_+$ for the particular setting of $p$ for all $\tu_3 \in  [0,\frac{8}{\kor^2}]$. This set precisely captures the set of all utilities which are consistent with the $\kor$ oracle $\oracle(\cdot\;;\kor)$. Since the optimal decision function $\f^* = \f_+$, this establishes the second part of the claim.\qed

\subsection{Proof of Theorem~\ref{thm:local-upper}}
Let us represent by $\simplex_\F$ the space of probability distributions over the function $\F$. The error of the estimator $\pfrob$ can then be upper bounded as
\begin{align*}
  \En[\gap(\pfrob, \F;\utrue)] &= \pinf_{\ld \in \simplex_\F} \sup_{\u'\in\; \uclass_{|\utrue}} \En_{\f}[ \gap(\f, \F; \u')]\\
  &= \pinf_{\ld \in \simplex_\F} \sup_{\u'\in\; \uclass_{|\utrue}} \sup_{\f' \in \F} \En_{\f} [\util(\f';\u') - \util(\f;\u')]\\
  &\stackrel{\1}{=} \sup_{\ld \in \simplex_{\F \times \uclass_{|\utrue}}} \pinf_{\f \in \F} \En_{(\f', \u')}\left[\util(\f';\u') - \util(\f;\u') \right],
\end{align*}
where the equality  $\1$ follows from an application of Sion's minimax theorem and the space $\simplex_{\F \times \uclass_{|\utrue}}$ denotes the space of all distributions over the joint space $\F \times \uclass_{|\utrue}$. Let us decompose the distribution $\ld = \ld_\u \cdot \ld_{\f|\u}$ where $\ld_\u$ represents the marginal distribution over the space $\uclass_{|\utrue}$ and $\ld_{\f|\u}$ denotes the conditional distirbution of sampling a function $\f \in \F$ given utility function $\u$. Denote by
\begin{align*}
  \f_u \defn \argmax_{\f\in \F}\util(\f;\u) \quad \text{and}\quad  \f_{p} \defn \argmax_{\f \in \F}\En_{\u \sim p}[\util(\f;\u)]
\end{align*}
   as the maximizers for the corresponding (expected) utility functions. Then, the excess risk
\begin{align*}
  \En[\gap(\pfrob, \F;\utrue)] &=  \sup_{\ld \in \simplex_{\F \times \uclass_{|\utrue}}} \pinf_{\f \in \F} \En_{(\f', \u')}\left[\util(\f';\u') - \util(\f;\u') \right] \\
  &= \sup_{\ld_\u}\sup_{\ld_{\f|\u}} \pinf_{\f \in \F} \left(\En_{\u' \sim \ld_u} \En_{\f' \sim \ld_{\f|\u'}}\left[\util(\f';\u')\right] - \En_{\u' \sim \ld_u}\left[\util(\f;\u') \right]\right)\\
  &\stackrel{\1}{=} \sup_{\ld_\u}\sup_{\ld_{\f|\u}} \left(\En_{\u' \sim \ld_u} \En_{\f' \sim \ld_{\f|\u'}}\left[\util(\f';\u')\right] - \En_{\u' \sim \ld_u}\left[\util(\f_{\ld_u};\u') \right]\right)\\
  &\stackrel{\2}{=} \sup_{\ld_\u}\left( \En_{\u' \sim \ld_u} \left[\util(\f_{\u'};\u') - \util(\f_{\ld_u};\u') \right]\right),
\end{align*}
where the inequality $\1$ follows from the fact that $\f_{\ld_u}$ maximizes the expected utility with respect to $\ld_u$ and $\2$ follows by noting that the maximizing distribution $\ld_{\f|\u'} = \ind[\f = \f_{\u'}]$. Noting that $\f_{\ld_u}$ is the maximizer corresponding to the distribution $\ld_u$, we have,
\begin{align*}
    \En[\gap(\pfrob, \F;\utrue)] &=   \sup_{\ld_\u}\left( \En_{\u' \sim \ld_u} \left[\util(\f_{\u'};\u') - \En_{\tilde{\u}\sim \ld_u}[\util(\f_{\tilde{\u}};\u')] + \En_{\tilde{\u}\sim \ld_u}[\util(\f_{\tilde{\u}};\u')] - \util(\f_{\ld_u};\u') \right]\right)\\
  &\leq \sup_{\ld_\u}\left( \En_{\u' \sim \ld_u} \left[\util(\f_{\u'};\u') - \En_{\tilde{\u}\sim \ld_u}[\util(\f_{\tilde{\u}};\u')] \right] \right)\\
  &\stackrel{\1}{\leq} \sup_{\u_1, \u_2 \in\; \uclass_{|\utrue}} \left( \util(\f_{\u_1};\u_1) - \util(\f_{\u_2};\u_1) \right),
\end{align*}
where the inequality $\1$ follows by upper bounding the expected deviation with a worst-case deviation. This establishes the required claim. \qed